\documentclass{article}
\usepackage{iclr2027_conference,times}

\iclrfinalcopy

\usepackage[utf8]{inputenc} % allow utf-8 input
\usepackage[T1]{fontenc}    % use 8-bit T1 fonts
\usepackage{wrapfig}
\usepackage{url}            % simple URL typesetting
\usepackage{booktabs}       % professional-quality tables
\usepackage{array}          % extended table support
\usepackage{amsfonts}       % blackboard math symbols
\usepackage{xcolor}         % colors

\usepackage{graphicx}       % figures
\usepackage{subcaption}     % subfigures
\usepackage{caption}        % caption customization
\usepackage{enumitem}       % customized lists
\usepackage{comment}        % comment environment

\usepackage{amsmath}
\usepackage{amssymb}
\usepackage{mathtools}
\usepackage{amsthm}
\usepackage{thmtools}

\usepackage{algorithm}
\usepackage{algorithmic}

\usepackage{hyperref}       % hyperlinks

\usepackage[capitalize,noabbrev]{cleveref}

\theoremstyle{plain}
\newtheorem{theorem}{Theorem}[section]

\newtheorem{lemma}{Lemma}[section]
\newtheorem{corollary}{Corollary}[section]

\theoremstyle{definition}

\newtheorem{assumption}{Assumption}

\theoremstyle{remark}
\newtheorem{remark}[theorem]{Remark}

\declaretheoremstyle[
  spaceabove=-8pt,
  spacebelow=8pt,
  headfont=\normalfont\itshape,
  postheadspace=0.5em,
  qed=\qedsymbol
]{mystyle}

\newcommand{\abs}[1]{\left\lvert #1 \right\rvert}

\newcommand{\brk}[1]{\left[ #1 \right]}

\newcommand{\bV}{\mathbb{V}}
\newcommand{\sg}{\operatorname{sg}}

\newcommand{\KL}{\mathrm{KL}}
\newcommand{\Qset}{\mathcal{Q}}
\newcommand{\Oset}{\mathcal{O}}
\newcommand{\diag}{\mathrm{diag}}

\title{Why GRPO Needs Normalization: A Local-Curvature Perspective on Adaptive Gradients}

\author{Cheng Ge\thanks{Author order is alphabetical denoting equal contributions.} \\
  Department of Aeronautics and Astronautics \\
  MIT\\
  \texttt{gec\_mike@mit.edu} \\
  \And
  Caitlyn Heqi Yin\footnotemark[1] \\
  Department of Statistics \\
  University of Wisconsin–Madison \\
  \texttt{hyin66@wisc.edu} \\
  \AND
  Hao Liang\thanks{Co-last authors} \\
  Department of Electrical and Computer Engineering \\
  University of Maryland, College Park \\
  \texttt{haoliang@umd.edu}
  \And
  Jiawei Zhang\footnotemark[2] \\
  Department of Computer Sciences \\
  University of Wisconsin–Madison \\
  \texttt{jzhang2924@wisc.edu} 
}

\begin{document}

\maketitle

\begin{abstract}
    Reinforcement learning (RL) has become a key driver of language model reasoning. Among RL algorithms, Group Relative Policy Optimization (GRPO) is the de facto standard, avoiding the need for a critic by using per-prompt baselines and variance normalization. Yet why and when this normalization helps remains unclear. In this work, we answer both questions through the lens of local curvature of the sequence-level policy gradient: standard deviation normalization implements an adaptive gradient. Theoretically, we prove that under mild conditions GRPO attains an improved convergence rate over unnormalized REINFORCE, with gains characterized by the average within-prompt reward standard deviation across prompts and iterations. We further introduce IS-GRPO, an importance-sampling variant of GRPO whose expected update remains aligned with the full gradient, and prove a convergence guarantee of the same form as REINFORCE's, which is tighter in practice during real training runs by a factor we measure. Empirically, on GSM8K and MATH we validate the curvature--variance link and measure the quantities that govern our bounds along real training runs. At 1.5B scale, per-prompt normalization outperforms both unnormalized and globally normalized baselines, with gains that emerge in a middle phase of training where per-prompt variances become heterogeneous. At 7B scale, the normalization schemes become statistically indistinguishable while IS-GRPO retains a small lead, in line with the theory's prediction that the gains require heterogeneity.
\end{abstract}

\section{Introduction}
Large language models (LLMs) have recently exhibited striking gains in multi-step reasoning, particularly when a lightweight reinforcement-learning (RL) stage is applied on top of a strong and pretrained base model. Among the many post-training recipes, Group Relative Policy Optimization (GRPO) has emerged as a practical, critic-free alternative that has powered some of the most visible reasoning systems \citep{shao2024deepseekmath}, where it consistently improves solution accuracy under tight compute budgets.  

While Proximal Policy Optimization (PPO) \citep{schulman2017proximal} remains a popular default in RLHF pipelines, it couples the policy with a learned value-function critic (and often GAE \citep{schulman2015high}). This increases memory footprint and implementation complexity, since the environment is a pretrained LLM and rewards arrive only at the end of whole sequences. This has renewed interest in critic-free policy-gradient methods that operate at the sequence level, such as REINFORCE-style methods, including  ReMax \citep{li2023remax} and RLOO \citep{ahmadian2024back}, which often match or outperform PPO for LLM alignment while being simpler and lighter-weight.

Classical REINFORCE \citep{williams1992simple} reduces gradient variance by subtracting a baseline (e.g., a running mean reward), yielding an unbiased estimator with lower variance \citep{greensmith2004variance}. This is textbook policy-gradient variance reduction. GRPO goes one step further: for each prompt, it samples multiple responses from the current policy, computes the group-mean reward as a baseline, and normalizes each response's update by the within-prompt reward standard deviation, which effectively uses a per-prompt z-scored advantage \citep{shao2024deepseekmath}. Empirically, this simple normalization has been repeatedly observed to stabilize optimization and improve sample-efficiency in LLM RL \citep{lambert2025reinforcement}. Yet, the underlying mechanism of this normalization step has not been theoretically clarified. A natural question arises:
\begin{center}
    \textit{Why and when does GRPO benefit from normalization?} 
\end{center}
In this work, we give a principled answer to this research question. Our key insight is that the reward variance of each question serves as an estimate of the local Lipschitz constant (curvature) of the policy gradient. Standard-deviation normalization therefore implements an adaptive gradient rule, scaling each update by an inverse-curvature proxy and assigning a larger step size to smoother prompts and a smaller step size to sharper prompts.  This mechanism improves stability and convergence precisely when (i) curvature is heterogeneous across prompts and evolves over training, and (ii) cross-prompt interference is controlled so that prompt-wise adaptive steps are not overwhelmed by interaction noise. 

Moreover, we show that the benefit of normalization is not tied to condition (ii): we introduce IS-GRPO, an importance-sampling variant of GRPO that selects each question in proportion to its current reward standard deviation while keeping the update positively aligned with the full gradient. We prove a strictly tighter overall bound under an explicit trajectory-heterogeneity condition for IS-GRPO that hold without any condition on cross-prompt interference. This perspective explains why GRPO outperforms unnormalized policy gradient methods in practice.

We establish concrete conditions under which normalization yields gains both theoretically and empirically. To summarize, our contributions are:
\vspace{-1ex}

\begin{itemize}[leftmargin=*]  
  \item \textbf{Curvature interpretation.} We provide a new perspective on GRPO as a principled adaptive gradient mechanism that adapts learning to per-prompt curvature.
  \vspace{-.5ex}
  \item \textbf{When normalization provably helps.} We show that under reasonable and mild assumptions, GRPO attains a faster convergence rate than unnormalized REINFORCE, with the speedup explicitly governed by the average per-prompt standard deviation over time. We further introduce IS-GRPO, an importance-sampling variant of GRPO, and  establish a similar advantage over REINFORCE without any condition on cross-prompt interference.
  \vspace{-.5ex}
  \item \textbf{When normalization empirically helps and when it does not}. On GSM8K and MATH, we validate the curvature–variance link and uncover three regimes for normalization: (I) early, near-orthogonal high-variance training where gains can be muted by noisy gradients; (II) a mid-stage with moderate variance and small pairwise gradient cosines where normalization delivers the strongest and most consistent improvements, especially on harder prompts; and (III) a late stage where the gap stops widening, coinciding with growing gradient coupling and, on easier prompts, with the accuracy ceiling.
\end{itemize}
\vspace{-1ex}
Our results explain why GRPO needs normalization. It is a principled adaptive gradient mechanism  whose benefits emerge under identifiable geometric and statistical conditions. This  provides a theoretical foundation for the empirical success of GRPO and offers broader insights into the design of critic-free RL algorithms for LLM training.
\vspace{-2ex}
\paragraph{Related works.}
We defer a detailed discussion of related work to Appendix~\ref{app:rel}, including prior work on REINFORCE-style policy gradient methods, GRPO and its variants, theoretical studies of GRPO, and RLVR.

\section{Preliminaries and Problem Settings}
\vspace{-1ex}
We begin by introducing the RLVR framework for LLM training and reviewing the GRPO algorithm, including the policy parametrization and update rules.
\vspace{-1ex}
\paragraph{Notation.}
For a finite set $\mathcal{X}$, we denote by $\Delta(\mathcal{X})$  the set of probability distributions over $\mathcal{X}$. By default, all vectors are treated as column vectors.  $\Vert \cdot \Vert = \Vert \cdot \Vert_2$ denotes the Euclidean norm for vectors and the spectral norm for matrices. For $v \in \mathbb{R}^m$, we use $\diag(v) \in \mathbb{R}^{m \times m}$ to denote the diagonal matrix with $v$ on its diagonal. We also use the shorthand notation $[m] \coloneqq \{1, \dots, m\}$ and $\mathcal{B}({v}, r) \coloneqq \{{x} \in \mathbb{R}^m \mid \Vert {x} - {v} \Vert_2 \le r\}$.

We say that a continuously differentiable function $f : \mathbb{R}^m \rightarrow \mathbb{R}^n$ is $C$-Lipschitz continuous if $\Vert \nabla f({x})\Vert \le C$ for all ${x} \in \mathbb{R}^m$. A function $f: \mathbb{R}^n \rightarrow \mathbb{R}$ is said to be $\rho$-weakly convex over $\mathcal{B}({v}, r)$ if $f + \frac{\rho}{2} \Vert \cdot \Vert^2$ is convex over $\mathcal{B}({v}, r)$. Moreover, $f$ is $L$-smooth over $\mathcal{B}({v}, r)$ if for all ${x_1},{x_2} \in \mathcal{B}({v}, r)$,  $\Vert \nabla f({x}_1) - \nabla f({x}_2) \Vert \le L \Vert {x}_1 - {x}_2\Vert$. An $L$-smooth function over $\mathcal{B}({v}, r)$ is automatically $L$-weakly convex on the same set. Throughout the paper, we do not distinguish between a Lipschitz-smooth function and a function with Lipschitz-continuous gradients, nor do we differentiate among the smoothness constant, the Lipschitz constant of the gradient, and the curvature. Finally, we use the terms question (answer) and prompt (output) interchangeably.

\subsection{Problem Setup}
\vspace{-1ex}
\paragraph{RLVR.} We consider a sequence-level RL with verifiable reward (RLVR) framework for  training LLMs. RLVR \citep{lambert2024tulu, guo2025deepseek, team2025kimi, wang2025reinforcement} has recently emerged as an effective paradigm for improving LLM reasoning performance, particularly in domains with deterministic and automatically verifiable rewards. Let $\mathcal{Q} = \{q_1, \dots, q_n\}$ denote the set of questions and $\mathcal{O} = \{o_1, \dots, o_K\}$  the set of possible output sequences. A predefined deterministic reward function $r: \mathcal{Q} \times \mathcal{O} \rightarrow \{0, 1\}$ evaluates the correctness of a response, where $r(q, o)=1$ if the $o$ is a correct response to $q$, and $r(q, o)=0$ otherwise. Given a question $q$, the LLM generates a response $o\sim\pi_{\theta}(q)$ according to a stochastic policy $\pi_{\theta}: \mathcal{Q} \rightarrow \Delta (\mathcal{O})$, parameterized by $\theta \in \Theta$. The objective is to learn a policy that maximizes the expected reward over the question set:
\begin{equation}
\vspace{-1ex}
\label{eqt:obj}
    J(\theta) := \frac{1}{n} \sum_{i = 1}^n J_{i}(\theta) = \frac{1}{n} \sum_{i = 1}^n \mathbb{E}_{o \sim \pi_{\theta}(\cdot \mid q_i)}[r(o, q_i)],
\vspace{-1ex}
\end{equation}
where $J_i(\theta) \coloneqq \mathbb{E}_{o \sim \pi_{\theta}(\cdot \mid q_i)}[r(o, q_i)]$ denotes the expect reward of policy $\pi_{\theta}$ on question $q_i$.
\begin{remark}
In LLM alignment settings such as RLHF, an additional KL penalty term is often introduced to mitigate over-optimization of the reward model:
\begin{equation*}
    \vspace{-1ex}
    J_{\KL}(\theta) := \frac{1}{n} \sum_{i = 1}^n \mathbb{E}_{o \sim \pi_{\theta}(\cdot \mid q_i)}[r(o, q_i)] - \beta \KL(\pi_{\theta} \Vert \pi_{\text{ref}}),
    \vspace{-1ex}
\end{equation*}
where $\beta \ge 0$ is the regularization parameter and $\pi_{\text{ref}}$ denotes a reference policy \citep{shao2024deepseekmath}. In contrast, for RLVR, recent studies \citep{chu2025gpg, hu2025open} have shown that removing the KL term can yield faster convergence and improved performance. Consequently, throughout this paper we focus on the unregularized objective, where $J_{\KL}(\theta)$ reduces to $J(\theta)$ for $\beta=0$. 
\end{remark}
\vspace{-1ex}
In this paper, we consider an on-policy setting with \textit{exact} gradient and randomly selected questions. At each iteration, a question $q_i$ is sampled \textit{uniformly} from $\mathcal{Q}$, and the corresponding gradient $\nabla J_i(\theta)$ is evaluated exactly.
\vspace{-2ex}
\paragraph{PPO.} PPO \citep{schulman2017proximal} has become a standard baseline for LLM alignment due to its stability and empirical effectiveness \citep{ouyang2022training, rafailov2023direct}. By combining policy gradient updates with a learned value-function critic and a clipped surrogate objective, PPO mitigates large policy updates and improves training robustness. Its objective function of PPO is given by:
\begin{equation}
\begin{aligned}
    \label{equ:ppo}
    J_{\textrm{PPO}}(\theta) &= \frac{1}{n} \sum_{i = 1}^n \mathbb{E}_{o \sim \pi_{\theta_{\textrm{old}}}(\cdot \mid q_i)}\Big[\min\Big(\gamma_i(o) A_i(o), \textrm{clip}(\gamma_i(o), 1-\epsilon, 1+\epsilon)A_i(o)\Big)\Big],
\end{aligned}
\end{equation}
where $\epsilon \in(0,1)$ is the clipping parameter, and $\gamma_i(o) \coloneqq \frac{\pi_{\theta}(o \mid q_i)}{\pi_{\theta_{\textrm{old}}}(o \mid q_i)}$ is the importance ratio between the current policy $\pi_{\theta}$ and the old policy $\pi_{\theta_{\textrm{old}}}$. The advantage $A_i(o)$ is calculated using Generalized Advantage Estimation with a learned critic. However, the reliance on a critic introduces additional computational overhead and implementation complexity. These considerations have motivated growing interest in critic-free alternatives such as REINFORCE.
\vspace{-2ex}
\paragraph{REINFORCE.} Unlike PPO, REINFORCE \citep{williams1992simple} appears well-suited for LLM alignment as it efficiently estimates the gradient with a single query of the language and reward model. It does not require training a value model, making it computationally more efficient than PPO. However, REINFORCE is known to suffer from high variance in its stochastic gradient estimates \citep{li2023remax}. In this paper, we focus on the \textit{exact} setting where the full gradient is computed for randomly selected questions. Under this setting, all critic-free PG methods, including REINFORCE, reduce to the vanilla PG update, as summarized in \Cref{alg:reinforce}.
\vspace{-2ex}

\begin{minipage}{0.45\textwidth}
\begin{algorithm}[H]
    \caption{Critic-free Policy Gradient Method}
    \label{alg:reinforce}
    \begin{algorithmic}
        \STATE \textbf{Input:} learning rate $\eta > 0$, initial parameter $\theta_0$.
        \FOR{$t = 0$ to $T-1$}
            \STATE Select $i_t$ uniformly at random from $\{1, \dots, n\}$
            \STATE $\theta_{t+1} \leftarrow \theta_{t}+ \eta \ \nabla J_{i_t}(\theta_{t})$
        \ENDFOR
        \STATE \textbf{Return:} final policy $\pi_{\theta_{T-1}}$
    \end{algorithmic}
\end{algorithm}
\end{minipage}
\hfill
\begin{minipage}{0.45\textwidth}
\begin{algorithm}[H]
    \caption{On-policy GRPO}
    \label{alg:grpo}
    \begin{algorithmic}
        \STATE \textbf{Input:} learning rate $\eta > 0$, initial parameter $\theta_0$.
        \FOR{$t = 0$ to $T-1$}
            \STATE Select $i_t$ uniformly at random from $\{1, \dots, n\}$
            \STATE $\theta_{t+1} \leftarrow \theta_{t}+ \eta \ 
            \frac{\nabla J_{i_t}(\theta_{t})}
            {\sqrt{\mathbb{V}[\pi_{\theta_{t}}(i_t)]}}$
        \ENDFOR
        \STATE \textbf{Return:} final policy $\pi_{\theta_{T-1}}$
    \end{algorithmic}
\end{algorithm}
\end{minipage}
\vspace{-1ex}

\subsection{GRPO}
GRPO, introduced in DeepSeek-Math \citep{shao2024deepseekmath} and DeepSeek-R1 \citep{guo2025deepseek}, builds on the computational efficiency of REINFORCE by eliminating the learned value-function  while substantially improving empirical performance. The GRPO objective closely resembles the PPO objective in Equation~(\ref{equ:ppo}):
\begin{equation}
\begin{aligned}
    \frac{1}{n} \sum_{i = 1}^n \mathbb{E}_{o \sim \pi_{\theta_{\textrm{old}}}(\cdot \mid q_i)}\Big[\min\Big(\gamma_i(o) A_i(o),
    \textrm{clip}(\gamma_i(o), 1-\epsilon, 1+\epsilon)A_i(o)\Big)\Big],
\end{aligned}
\end{equation}
differing only in  the advantage term $A_i(o)$. Specifically, GRPO defines the advantage as:
\begin{equation*}
    A_i(o) \coloneqq \frac{r(q_i, o)-\mathbb{E}_{o^{\prime} \sim \pi_{\theta_{\textrm{old}}}(\cdot \mid q_i)} \brk{r\left(q_i, o^{\prime}\right)}}{\sqrt{\bV_{o^{\prime} \sim \pi_{\theta_{\textrm{old}}}(\cdot \mid q)} \brk{r\left(q_i, o^{\prime}\right)} }},
\end{equation*} 
which uses the group mean reward as a baseline, and normalizes by the within-prompt reward standard deviation.  For theoretical analysis, we impose the assumption that each question admits a unique correct answer, a condition commonly adopted in the study of RL algorithms \citep{mei2020global, mei2023stochastic, lin2025rethinking}.
\begin{assumption}[Unique correct answer]
    \label{ass:ass_1}
    For any $q \in \Qset$, there exists a unique $o^*(q) \in \Oset$ such that $r(q,o^*(q)) = 1$.
\end{assumption}
Under Assumption~\ref{ass:ass_1}, let $a_i$ denote the index of the correct answer for question $q_i \in \Qset$, so that $r(q_i, o_{a_i})=1$. We use ${r}_i \in \mathbb{R}^K$ to denote the \textit{one-hot} reward vector for question $q_i$, with $[{r}_i]_j = r(q_i, o_j), \forall j \in [K]$. We consider the on-policy scenario where $\pi_{\theta} = \pi_{\theta_{\textrm{old}}}$. Let $\pi_{\theta}^*(i) \coloneqq \pi_{\theta}(o_{a_i}\mid q_i)$ be the success probability of policy $\pi_{\theta}$ on question $q_i$, and let $\mathbb{V}[\pi_{\theta}(i)]=\pi_{\theta}^*(i) \Big(1 - \pi_{\theta}^*(i)\Big)$ denote  its Bernoulli reward variance. Under this setting, the importance ratio  $\gamma_i(o)=1$, and the GRPO advantage simplifies to
\begin{equation}
    A_i(o) = \frac{r(q_i, o) - \sg(\pi_{\theta}^*(i))}{\sg(\sqrt{\mathbb{V}[\pi_{\theta}(i)]})},
\end{equation}
where $\sg(\cdot)$ denotes the stop-gradient operator, which treats its argument as a constant during backpropagation. The GRPO objective can be simplified as:
\begin{align*}
    J_{\textrm{GRPO}}(\theta) &= \frac{1}{n} \sum_{i = 1}^n J^i_{\textrm{GRPO}}(\theta) := \frac{1}{n} \sum_{i = 1}^n \mathbb{E}_{o \sim \pi_{\theta}}\Big[A_i(o)\Big].
\end{align*}
We further denote by $\pi_{\theta}(i) \in \mathbb{R}^K$ the probability vector for $\pi_{\theta}$ on question $i$, with $[\pi_{\theta}(i)]_j \coloneqq \pi_{\theta}(o_j \mid q_i), \forall j \in [K]$. Since subtracting a constant baseline does not affect the gradient of the objective, we therefore refer to the resulting algorithm as \textit{(on-policy) GRPO}. Our key observation is that the variance normalization in  GRPO implicitly implements an adaptive step size. For all $i\in[n]$,
\begin{equation}\label{eq:eq5}
    \begin{aligned}
            \nabla J^i_{\textrm{GRPO}}(\theta) &= \mathbb{E}_{o \sim \pi_{\theta}}[A_i(o) \nabla \ln \pi_{\theta}(o \mid q_i)]
            = \frac{\mathbb{E}_{o \sim \pi_{\theta}}[r(q_i, o) \nabla \ln \pi_{\theta}(o \mid q_i)]}{\sqrt{\mathbb{V}[\pi_{\theta}(i)]}}=\frac{\nabla J_i(\theta)}{\sqrt{\mathbb{V}[\pi_{\theta}(i)]}}.
    \end{aligned}
\end{equation}
The first and last equalities follow from the policy gradient theorem \citep{sutton1998reinforcement}. The second equality holds because subtracting a constant baseline does not affect the gradient calculation. 
The pseudo-code for GRPO  is provided in \Cref{alg:grpo}.

\vspace{-2ex}

\paragraph{Policy parametrization.} We focus on the \emph{log-linear policy parametrization}, which has been widely studied in the analysis of policy gradient methods \citep{agarwal2021theory, yuan2022linear}. Specifically,  for each question-output pair $(q_i, o_j)$, we assume the existence of a fixed feature vector ${x}_{i, j} \in \mathbb{R}^d$, and define the policy as $\pi_{\theta}(o_j \mid q_i) := \frac{\exp({x}_{i, j}^{\top}\theta)}{\sum_{l = 1}^{K} \exp({x}_{i, l}^{\top} \theta)}$.

We denote by $X_i \in \mathbb{R}^{K \times d}$ the feature matrix associated with question $q_i$: $X_i \coloneqq \begin{pmatrix}
        {x}_{i, 1},  \cdots , {x}_{i, K}
    \end{pmatrix}^{\top}.$
The explicit update rules for REINFORCE (\Cref{alg:reinforce}) and GRPO (\Cref{alg:grpo}) under this parametrization are provided in Appendix \ref{app:grad_upd}.

\vspace{-2ex}

\section{Theoretical Results}\label{sec:theoretical}

\vspace{-1ex}

In this section, we present a convergence analysis of REINFORCE and GRPO. We show that GRPO can achieve provably faster convergence than REINFORCE under our assumptions. Because PG with linear function approximation may fail to converge to the optimal policy in general, our analysis focuses on convergence rates toward stationary points.
\vspace{-1ex}
\subsection{Local Smoothness of Objective Function}
\label{subsec:local_smoothness}
We begin by relating reward variance to the smoothness of of the per-prompt objective. Let $X_{\max} := \max_{i \in [n]} \Vert X_i \Vert$. 
\begin{lemma}
    \label{lemma:local_smooth}
    Under \Cref{ass:ass_1}, for all $i \in [n]$ and $\theta \in \mathbb{R}^d$, we have $\Vert \nabla^2 J_i(\theta)\Vert \le
         4 X_{\max}^2 \cdot \bV[\pi_{\theta}(i)]$.
    Moreover, $\Vert \nabla^2 J_i(\theta)\Vert \le X_{\max}^2$,
    and hence $J_i(\theta)$ is $X_{\max}^2$-smooth on $\mathbb{R}^d$.
    \vspace{-1ex}
\end{lemma}
The proof is provided in Appendix \ref{appsub:lemma_1}. This lemma shows that the local smoothness constant of $J_i(\theta)$ is directly controlled by the reward variance of $q_i$ under the current policy $\pi_\theta$. 

For deterministic gradient descent on an $L$-smooth function, a step size of $1/L$ is a standard choice \citep{garrigos2023handbook}. Lemma~\ref{lemma:local_smooth} therefore motivates adapting the step size to the local smoothness constant of each per-prompt objective $J_i(\theta)$, which varies across prompts. GRPO achieves exactly this by normalizing updates with reward variance, which implicitly implements an adaptive step size matched to the local curvature of each prompt.
\vspace{-1ex}
\subsection{Orthogonal Representation}

\label{subsec:orthogonal_assumption}
We first impose the following assumption on the bound of within-prompt Bernoulli variance at every step.
\vspace{-1ex}
\begin{assumption}[Bounded variance]
\label{ass:ass_3}
For each $i \in [n]$, there exists a positive sequence $\{C_{i, t}\}_{t = 1}^{\infty}$ 
\[
\sqrt{\bV[\pi_{\theta_t}(i)]} \le C_{i, t} \le \frac{1}{2}.
\]
\end{assumption}
\vspace{-1ex}
We further denote $\mu_{k, t} \coloneqq \frac{1}{n}\sum_{i=1}^n C^k_{i, t}.$ To extend \Cref{lemma:local_smooth} from the per-prompt objectives $J_i(\theta)$ to the averaged objective $J(\theta) = \frac{1}{n}\sum_{i=1}^n J_i(\theta)$, we must control how gradients of different prompts interact during optimization; in particular, they should not interfere destructively with one another. Empirically, as we show in \Cref{sec:emperical}, gradients associated with different prompts are nearly orthogonal. Motivated by this observation, we introduce the following assumption.

\begin{assumption}[Orthogonal representation]
    \label{ass:ass_2}
    For all $i,j \in [n]$ with $i \neq j$, we have $X_i X_j^{\top} = {0}$.
\end{assumption}

This assumption ensures that the gradients associated with different prompts are mutually orthogonal, allowing their effects on the optimization dynamics to be analyzed independently for both REINFORCE and GRPO.

\subsection{Convergence Result under the Orthogonal Representation Assumption}

\label{subsec:convergence}
Before presenting the convergence results for \Cref{alg:reinforce} and~\ref{alg:grpo}, we further characterize the smoothness properties of $J_i$ at each iteration.
\begin{lemma}
    \label{lemma:lipschitz}
    Under \Cref{ass:ass_1}, for all $i \in [n]$, we know
    \begin{equation*}
        \Vert \nabla J_i(\theta) \Vert \le 2 \bV[\pi_{\theta}(i)] X_{\max}.
    \end{equation*}
\end{lemma}
\begin{corollary}
    Under \Cref{ass:ass_1}, for all $i \in [n]$, $J_i(\theta)$ is $\frac{1}{2} X_{\max}$-Lipschitz over $\mathbb{R}^d$.
\end{corollary}
\begin{lemma}[Non-uniform local smoothness]
    \label{lemma:step_smooth}
    Under \Cref{ass:ass_1}, for all $i \in [n]$ and $\theta \in \mathbb{R}^d$, we have $\Vert \nabla^2 J_{i}(\theta') \Vert \le \frac{5}{2} X_{\max}^2 \cdot \sqrt{\bV[\pi_{\theta}(i)]}$ for every $\theta' \in \mathcal{B}(\theta, \frac{1}{X_{\max}} \cdot \sqrt{\bV[\pi_{\theta}(i)]})$. Consequently, $J_i$ is $\frac{5}{2} X_{\max}^2 \cdot \sqrt{\bV[\pi_{\theta}(i)]}$-smooth within this neighborhood.
    \vspace{-1ex}
\end{lemma}
The proofs are provided in \Cref{appsub:lemma_2} and~\ref{appsub:lemma_3}. We now present our two main theorems, which establish the convergence guarantees for  REINFORCE and GRPO, respectively. Their detailed proofs are deferred in Appendix \ref{app:thm}.

\begin{theorem}[Convergence rate of REINFORCE per question]
    \label{thm:reinforce}
    Under \Cref{ass:ass_1} and \Cref{ass:ass_2}, with step size $\eta = \tfrac{1}{X_{\max}^2}$, \Cref{alg:reinforce} satisfies
\[
\mathbb{E}[J_i(\theta_t)] - \mathbb{E}[J_i(\theta_{t-1})]
\;\ge\; \frac{1}{2nX_{\max}^2} \,\mathbb{E}[\|\nabla J_i(\theta_{t-1})\|^2]. \vspace{-2ex}
\]
Moreover, 
\begin{align*}
    \min_{t \in \{0, 1, \dots, T-1\}} \mathbb{E}[\Vert\nabla J_i(\theta_t)\Vert^2] \le \frac{2n(1 - \pi_{\theta_0}^*(l))X_{\max}^2}{T}. 
\vspace{-2ex}
\end{align*}
\vspace{-2ex}
\end{theorem}

\begin{theorem}[Convergence rate of GRPO per question]
    \label{thm:grpo}
    Under \Cref{ass:ass_1}-\ref{ass:ass_2} with the step size $\eta = \frac{2}{5X_{\max}^2}$, \Cref{alg:grpo} satisfies
    \[
    \mathbb{E}[J_i(\theta_t)] - \mathbb{E}[J_i(\theta_{t-1})]
    \;\ge\; \frac{\mathbb{E}[\|\nabla J_i(\theta_{t-1})\|^2]}{5nX_{\max}^2\,C_{i, t}}. \vspace{-2ex}
    \]
    Moreover, we have
    \begin{align*}
        \min_{t \in \{0, 1, \dots, T-1\}} \mathbb{E}[\Vert\nabla J_i(\theta_t)\Vert^2] \le \frac{2n(1 - \pi_{\theta_0}^*(l))X_{\max}^2}{T} \cdot C(i, T), 
    \end{align*}
    where $C(i, T) \coloneqq \frac{5}{2T}\sum_{t=0}^{T-1} C_{i, t}$.
\end{theorem}

\begin{remark}
Comparing Theorems~\ref{thm:reinforce} and~\ref{thm:grpo}, we observe that GRPO achieves a strictly better average convergence bound than REINFORCE, provided that the following constant factor satisfies
\vspace{-1ex}
\begin{align}
    \label{equ:normalization}
    C(T):= \frac{1}{n}\sum_{i = 1}^n C(i, T) = \frac{5}{2T} \sum_{t = 0}^{T-1} \mu_{1, t} < 1.
\end{align}
Here, $C(T)$ equals $\tfrac{5}{2}$ times the average within-prompt Bernoulli standard deviation over prompts and iterations. 

In practice, $C(T)$ is typically much smaller than $1$ when the question set contains heterogeneous prompts with varying levels of difficulty. We further validate $C(T)<1$ empirically in Section~\ref{sec:emperical}. Intuitively, GRPO automatically assigns smaller effective step sizes to prompts with higher curvature, yielding faster convergence on average. Similar improvements can also be expected in settings where the curvature varies significantly across iterations.
\end{remark}

\subsection{Removing the Orthogonality Assumption}
\label{subsec:isgrpo}
The preceding results isolate the benefit of variance normalization, but their proofs rely on \Cref{ass:ass_2} to decouple the per-question objectives. We now remove this requirement and analyze the mean objective $J(\theta) = \frac{1}{n}\sum_{i=1}^n J_i(\theta)$ directly. To this end, we introduce IS-GRPO (\Cref{alg:isgrpo}), a variant of GRPO whose question-selection rule is based on importance sampling, a classical technique for non-uniform sampling in stochastic gradient methods  \citep{needell2014stochastic, zhao2015stochastic}. Instead of the uniform question sampling in \Cref{alg:reinforce,alg:grpo}, IS-GRPO samples each question in proportion to its current reward standard deviation and rescales the corresponding gradient, so that the update remains positively aligned with the full gradient $\nabla J(\theta)$.

\begin{algorithm}[t]
    \caption{\textbf{I}mportance \textbf{S}ampling \textbf{GRPO}}
    \label{alg:isgrpo}
    \begin{algorithmic}
        \STATE \textbf{Input:} learning rate $\eta > 0$, initial parameter $\theta_0$.
        \FOR{$t = 0$ to $T-1$}
            \STATE Form $p_{t,i} = \frac{\sqrt{\mathbb{V}[\pi_{\theta_{t}}(i)]}}{\sum_{j=1}^n \sqrt{\mathbb{V}[\pi_{\theta_{t}}(j)]}}$ for $i \in [n]$
            \STATE Sample $i_t \sim p_t$ and compute $g_{t} := \frac{\nabla J_{i_t}(\theta_{t})}{2\sqrt{\mathbb{V}[\pi_{\theta_{t}}(i_t)]}}$
            \STATE $\theta_{t+1} \leftarrow \theta_{t} + \eta\, g_{t}$
        \ENDFOR
        \STATE \textbf{Return:} final policy $\pi_{\theta_{T-1}}$
    \end{algorithmic}
\end{algorithm}

For the sampled question $q_{i_t}$, the update direction $g_{t}$ defined in \Cref{alg:isgrpo} coincides with the normalized GRPO gradient in \eqref{eq:eq5} up to a constant scalar of $2$. Thus, importance sampling changes only how questions are selected, while each selected question is still updated with GRPO's normalized gradient. The convergence rates of the two algorithms are governed by the trajectory coefficients
\begin{equation*}
    D(T) \coloneqq \frac{2}{T}\sum_{t=0}^{T-1} \mu_{2, t}\,\mu_{4, t}
    \qquad \text{and} \qquad
    D_{\mathrm{IS}}(T) \coloneqq \frac{\sum_{t=0}^{T-1} \mu_{4, t}}{\sum_{t = 0}^{T-1} \frac{1}{\mu_{1, t}}} ,
\end{equation*}
associated with \Cref{alg:reinforce} and \Cref{alg:isgrpo}, respectively. Our first theorem provides the baseline guarantee for \Cref{alg:reinforce} on the mean objective.
 
\begin{theorem}[Convergence of REINFORCE without orthogonality]
    \label{thm:reinforce_is}
    Suppose \Cref{ass:ass_1} and \Cref{ass:ass_3} hold. For every step size $0 < \eta \le \frac{1}{2X_{\max}^2}$, \Cref{alg:reinforce} satisfies
    \begin{equation}\label{eq:reinforce_avg}
        \min_{t \in \{0, \dots, T-1\}} \mathbb{E}\Vert \nabla J(\theta_t)\Vert^2
        \;\le\; \frac{1 - J(\theta_0)}{\eta T} + 8\,\eta\,X_{\max}^4\, D(T).
    \end{equation}
\end{theorem}
\Cref{thm:reinforce_is} extends the per-question guarantee of \Cref{thm:reinforce} to the mean objective without \Cref{ass:ass_2}. Our second theorem shows that \Cref{alg:isgrpo} satisfies a guarantee of a similar form, with $D(T)$ replaced by $D_{\mathrm{IS}}(T)$.
 
\begin{theorem}[Convergence of IS-GRPO without orthogonality]
    \label{thm:isgrpo}
    Suppose \Cref{ass:ass_1} and \Cref{ass:ass_3} hold. For every step size $0 < \eta \le \frac{1}{2X_{\max}^2}$, \Cref{alg:isgrpo} satisfies
    \begin{equation}\label{eq:isgrpo_avg}
        \min_{t \in \{0, \dots, T-1\}} \mathbb{E}\Vert \nabla J(\theta_t)\Vert^2
        \;\le\; \frac{1 - J(\theta_0)}{\eta T} \cdot \frac{4C(T)}{5} + 8\,\eta\,X_{\max}^4\, D_{\mathrm{IS}}(T),
    \end{equation}
    where $C(T)$ is defined in \eqref{equ:normalization}.
\end{theorem}
All proofs for this subsection are deferred to \Cref{app:isgrpo}.

\begin{remark}
    Since both theorems share the same step-size range, the comparison rests on the factor $\tfrac{4}{5}C(T)$ and on $D_{\mathrm{IS}}(T)$ versus $D(T)$. For the first term, we already showed in \Cref{subsec:convergence} that $C(T) < 1$, which leads to $\frac{4}{5}C(T) < \frac{4}{5}$. Thus, the first term of \Cref{thm:isgrpo} is strictly smaller than that of \Cref{thm:reinforce_is}. 
    
    For the second term, $D_{\mathrm{IS}}(T) < D(T)$ is satisfied when $C_{i,t}$ vary substantially across iterations and some of the $\mu_{1,t}$ are close to zero, since such iterations contribute terms of order $1/\mu_{1,t}$ to the left-hand side. We validate $D_{\mathrm{IS}}(T) < D(T)$ empirically in Section~\ref{sec:emperical}.

    Finally, with the step size $\eta\asymp 1/\sqrt{T}$, the two algorithms converge at the same rate $O(1 /\sqrt{T})$. However, IS-GRPO attains this rate with strictly smaller constants.
\end{remark}
\vspace{-2ex}

\section{Empirical Studies}
\label{sec:emperical}
\vspace{-2ex}
The analysis of \Cref{sec:theoretical} is stated for a log-linear policy updated by plain gradient ascent, with exact per-prompt gradients and exactly known reward variances. Practical GRPO differs on every one of these points: gradients are estimated from $G=8$ rollouts per prompt, the optimizer is AdamW, and each prompt's standard deviation is itself estimated from the same eight samples. We therefore do not expect the constants in \Cref{thm:reinforce,thm:grpo,thm:reinforce_is,thm:isgrpo} to transfer literally. What the experiments test is (i)~whether the assumptions and conditions on which the theory rests are met by real training runs (\Cref{subsec:curvature_variance,subsec:coefficients}), and (ii)~whether the qualitative predictions of the theory hold: how the methods should be ordered, and under which conditions the gaps between them should appear or vanish (\Cref{subsec:comparisons}).

\subsection{Validation of Local Curvature--Variance Connection}
\label{subsec:curvature_variance}
\begin{minipage}[t]{0.58\linewidth}
We use the diagonal Fisher information as a curvature proxy~\citep{liu2023sophia}, with estimation details provided in Appendix~\ref{app:fisher_method}. As shown in Table~\ref{tab:fisher_correlation}, curvature and reward variance are positively associated within the same iteration, but \emph{not} across different iterations. This finding is consistent with Lemma~\ref{lemma:local_smooth}. Since \Cref{lemma:local_smooth} bounds rather than equates the two, we test only for a positive association.

\end{minipage}\hfill
\begin{minipage}[t]{0.40\linewidth}
  \centering
  \vspace{-10pt}
  \captionof{table}{Fisher--variance correlation.}
  \label{tab:fisher_correlation}
  \small
  \setlength{\tabcolsep}{3pt}
  \begin{tabular}{lccc}
    \toprule
    Lag & $r$ & $\rho$ & Sig.\ \\
    \midrule
    Same iter.  &  .342 & .41  & Yes \\
    Diff.\ iter. & $-$.028 & $-$.015 & No \\
    \bottomrule
  \end{tabular}
  \vspace{4pt} 
  \\
  \textbf{Same iter.}: curvature and variance measured at the same checkpoint. \\
  \textbf{Diff. iter.}: curvature at step $t$, variance at step $t' \neq t$.
\end{minipage}
\vspace{-1ex}

\noindent
\begin{minipage}[t]{0.51\linewidth}
    \textbf{Does the orthogonality assumption hold?} \Cref{ass:ass_2} requires that per-prompt gradients are orthogonal (\Cref{lemma:auxiliary_1}). Tracking pairwise gradient cosine similarities on GSM8K, we observe that they are near-orthogonal at initialization but grow increasingly coupled during training, so the gradients are no longer nearly orthogonal by the end (\Cref{fig:3phases}, \Cref{app:f2}). This observation motivates IS-GRPO, as \Cref{thm:reinforce_is,thm:isgrpo} require no orthogonality assumption and remain valid once interference grows.
\end{minipage}
\hfill
\begin{minipage}[t]{0.46\linewidth}
    \vspace{4pt}
    \centering
    \includegraphics[width=\linewidth]{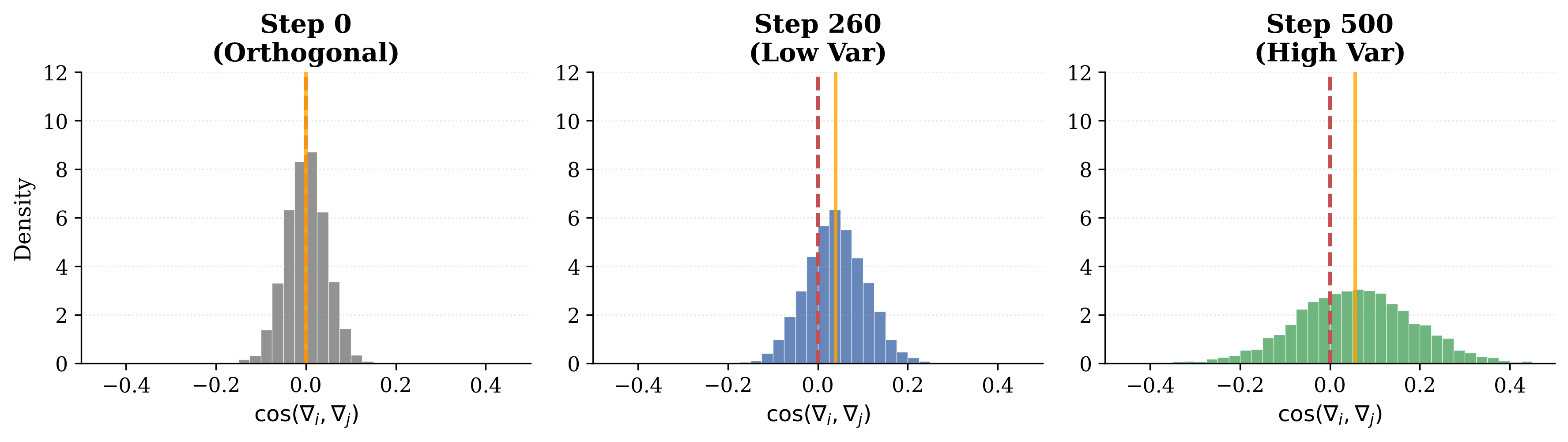}
    \captionof{figure}{Pairwise gradient cosine similarity during GRPO training (1.5B, GSM8K, 4{,}950 pairs). Red dashed: $\cos\!=\!0$; orange: mean.}
    \label{fig:3phases}
\end{minipage}

\subsection{Comparisons on LLM Reasoning Tasks}
\label{subsec:comparisons}

\paragraph{Methods.}
We compare four update rules that share the same rollouts, the same per-prompt mean baseline and the same optimizer: (1)~\textbf{per-prompt std} (standard GRPO), which divides by the within-prompt reward std; (2)~\textbf{global std}, which divides by the mini-batch-wide std, so that all prompts in a batch share one scaling factor; (3)~\textbf{REINFORCE}, with mean subtraction only; and, at 7B scale, (4)~\textbf{IS-GRPO}, which combines per-prompt normalization with the question-selection rule of \Cref{alg:isgrpo}. In practice, IS-GRPO replaces $\sqrt{\bV[\pi_{\theta}(i)]}$ in \Cref{alg:isgrpo} by a running per-prompt estimate $s_i$ and normalizes advantages by the within-group std as in GRPO (\Cref{app:setup_details}).

\paragraph{What the theory predicts.}
\Cref{thm:grpo,thm:isgrpo} guarantee that, in the exact-gradient model, per-prompt GRPO (under \Cref{ass:ass_2}) and IS-GRPO (without it) have smaller convergence constants than REINFORCE once $C(T)<1$.
%Two issues effect how this transfers to practice. First, the gain of the exact-variance analysis concentrates on prompts with $\bV\to0$, whose factor $1/\sqrt{\bV}$ is unbounded; with $G=8$ rollouts such prompts mostly yield all-correct or all-incorrect groups and receive zero advantage under every scheme, so $C(T)<1$ is a necessary condition rather than a measure of the practical gain. Second, our theorems do not analyze global std; the curvature interpretation of \Cref{subsec:local_smoothness} suggests that a single batch-wide factor cannot match the differing smoothness constants of individual prompts, so per-prompt normalization should improve on global std only when the factors $1/\sqrt{\bV[\pi_\theta(i)]}$ of the prompts that receive a gradient differ. 
We expect per-prompt normalization to help while active prompts are heterogeneous, and to bring little gain when few, similar prompts remain active. IS-GRPO's selection rule only needs to distinguish zero from non-zero variance and is not subject to this limitation.

\vspace{-1ex}

\subsubsection{Small Scale Experiments: Accuracy from Low to Relatively High}
\label{subsec:4.5}

\begin{wrapfigure}{l}{0.56\textwidth}
  \vspace{-12pt}
  \centering
  \includegraphics[width=0.4\linewidth]{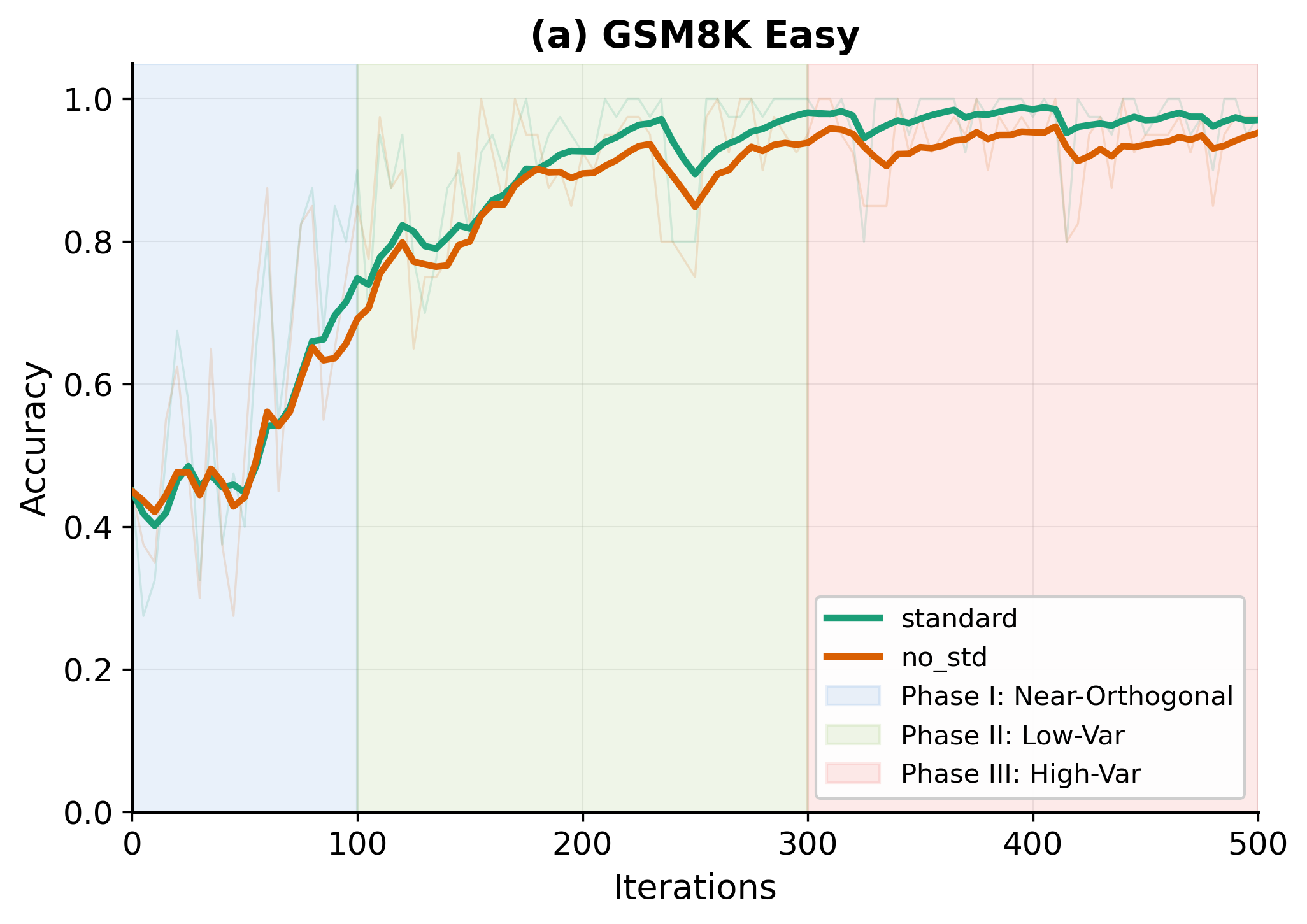}\hfill
  \includegraphics[width=0.4\linewidth]{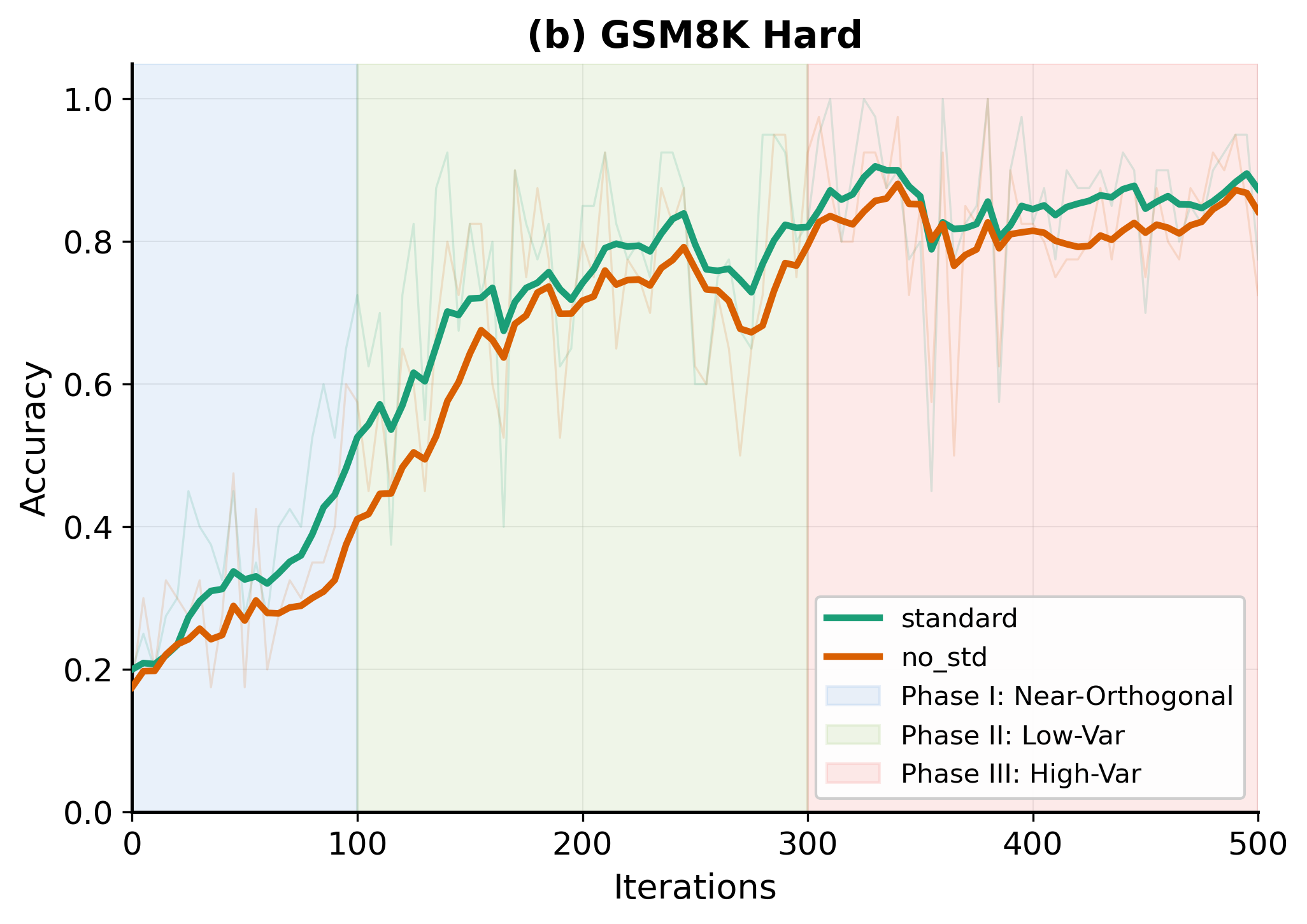}
  \caption{Training accuracy on GSM8K (1.5B) with phase annotations; left: Easy, right: Hard.
  \textit{standard}: per-prompt normalization; \textit{no\_std}: unnormalized (global std omitted for readability; see \Cref{tab:final_accuracy_full}).}
  % Shaded: $\pm1$ std over 3 seeds; EMA $\alpha\!=\!0.18$.}
  \label{fig:norm}
  \vspace{-12pt}
\end{wrapfigure}
\Cref{fig:norm} shows training accuracy on GSM8K with Qwen2.5-Math-1.5B; \Cref{tab:final_accuracy_full} reports final numbers. Per-prompt std consistently outperforms both baselines. 

\textbf{Phase analysis.}
The dynamics fall into three phases. In \textbf{Phase~I (0--100)} all methods overlap: with $G=8$ rollouts many groups are all-correct or all-incorrect and receive zero advantage, while the active prompts have similar success probabilities, so per-prompt and global normalization coincide. In \textbf{Phase~II (100--300)}, as accuracies rise heterogeneously, $c_t$ falls and the scaling factors of active prompts spread out; per-prompt std opens a widening gap, reaching $\approx$\,3\,pp (Easy) and $\approx$\,7\,pp (Hard) over REINFORCE by iteration 300. In \textbf{Phase~III (300--500)} the gap narrows but persists (\Cref{tab:final_accuracy_full}). This coincides with a rise of the gradient-cosine std and, on Easy, with accuracies approaching the ceiling, which by itself limits the room for any gap.

\subsubsection{Larger Scale Experiments: 7B Model}
\label{subsec:7b_experiments}

\begin{wrapfigure}{r}{0.49\textwidth}
  \vspace{-12pt}
  \centering
  \begin{minipage}[b]{0.49\linewidth}
    \centering
    \includegraphics[width=\linewidth]{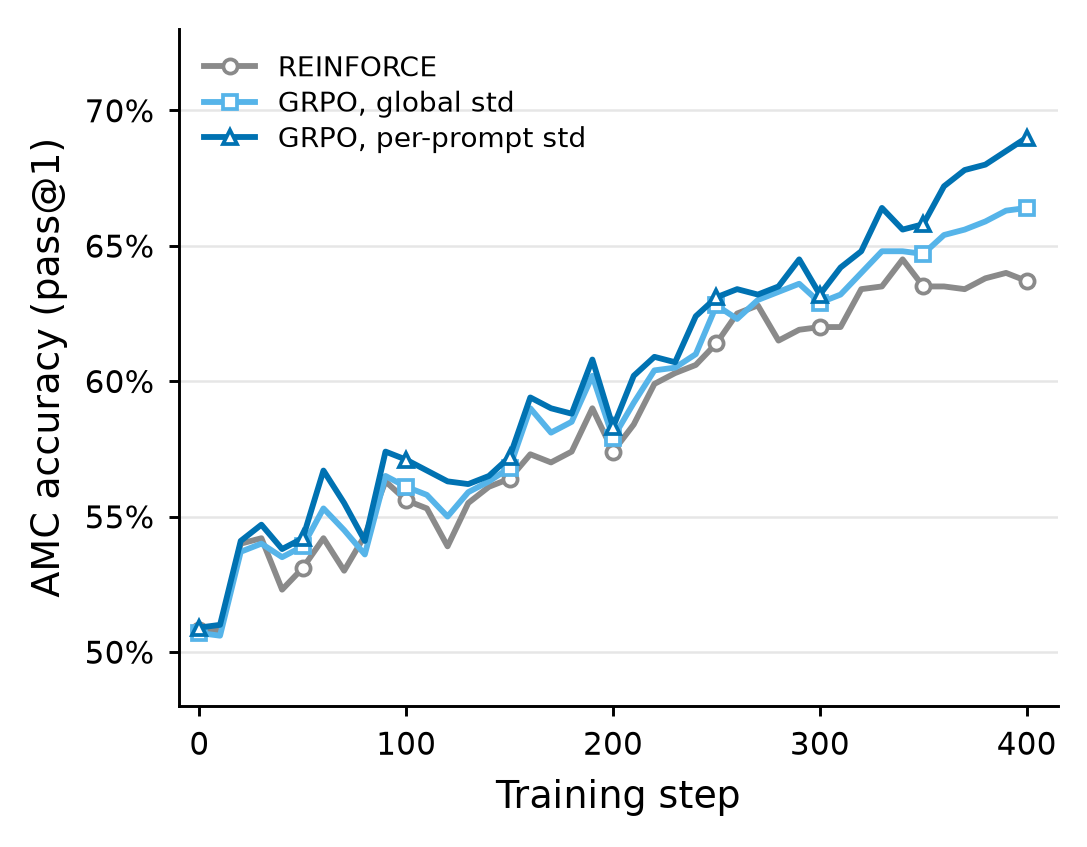}\\[-2pt]
    {\small (a) MATH train; AMC test}
  \end{minipage}\hfill
  \begin{minipage}[b]{0.49\linewidth}
    \centering
    \includegraphics[width=\linewidth]{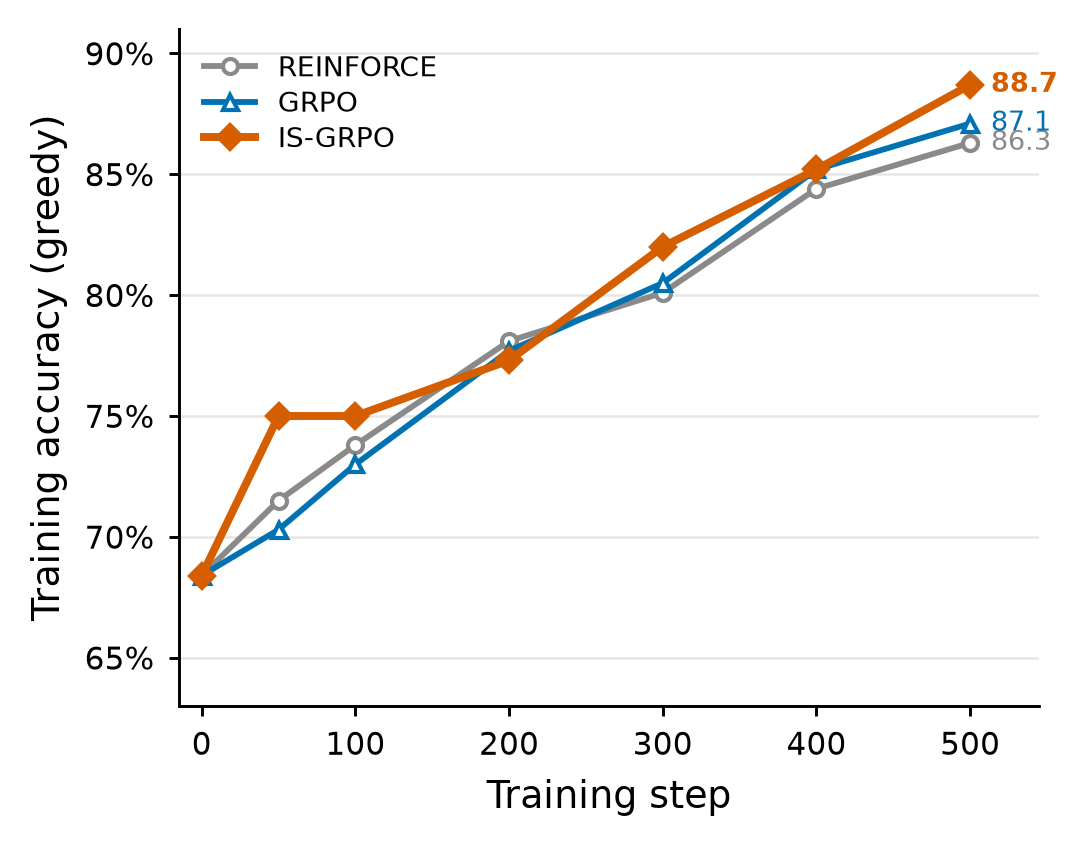}\\[-2pt]
    {\small (b) MATH train}
  \end{minipage}
  \caption{7B results (Qwen2.5-Math-7B). (a) Consistent with normalization helping when many prompts remain unsolved (AMC held-out). (b) When most prompts are solved, the normalization schemes coincide and IS-GRPO leads.}
  \label{fig:acc_7b}
  \vspace{-12pt}
\end{wrapfigure}
We run Qwen2.5-Math-7B in two regimes that differ in how much of the training set is already solved (\Cref{app:7b_details}). \\
\textbf{Harder training set.} Trained on MATH Level~3--5 and evaluated on AMC (\Cref{fig:acc_7b}a), the schemes follow the 1.5B pattern: they overlap early, after which per-prompt std pulls ahead of global std, which stays ahead of REINFORCE. On a second, harder held-out benchmark (AIME 2024), we observe a slightly positive result for normalization (\Cref{app:7b_details}).

\textbf{Largely solved training set.} Trained on MATH Level~1--5 and evaluated on the training prompts, the proxy of $J(\theta)$ (\Cref{fig:acc_7b}b, three seeds), the three normalization schemes end within one across-seed standard deviation, while IS-GRPO leads REINFORCE, most visibly early in training. This setting resembles Phase~III from the start: most groups carry no gradient and the remaining prompts are similar. 
%The decoupling on which \Cref{thm:grpo} relies may also weaken here as interference grows, though we measured interference only at 1.5B (\Cref{app:f2}), not at 7B. 
Among the features that make per-prompt normalization coincide with global std in this run, resolution and optimizer invariance ((i), (ii) in \Cref{app:7b_details}) hold equally at 1.5B, where per-prompt normalization does help; what distinguishes this run is (iii), the high base accuracy that leaves few, similar active prompts. IS-GRPO is affected by neither: its guarantee does not require orthogonality, and its selection rule stops drawing solved prompts. 
\vspace{-2ex}
\subsection{Trajectory Coefficients in Practice}
\label{subsec:coefficients}

\begin{wrapfigure}{l}{0.48\textwidth}
  \vspace{-12pt}
  \centering
  \includegraphics[width=\linewidth]{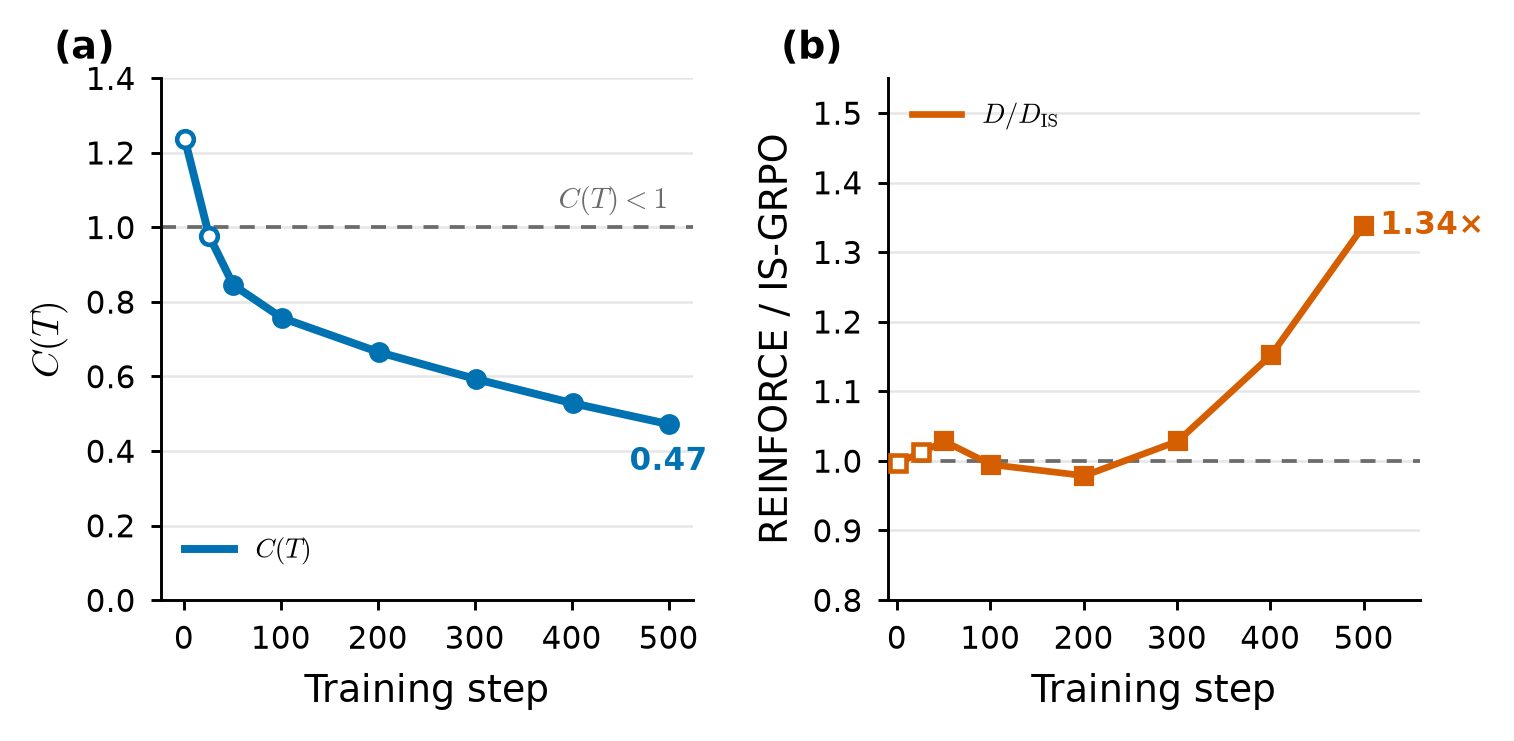}
  \caption{Trajectory coefficients along the 7B IS-GRPO run. (a)~Running average $C(T)$; the dashed line is the speed-up condition~\eqref{equ:normalization}. (b)~Ratio $D(T)/D_{\mathrm{IS}}(T)$.}
  \label{fig:coef_7b}
  \vspace{-12pt}
\end{wrapfigure}
Along the 7B IS-GRPO run, C(T) falls below 1 early in training and keeps decreasing, ending at $0.47$ (\Cref{fig:coef_7b}a), so condition~\eqref{equ:normalization} holds for essentially the whole run. $D(T)$ and $D_{\mathrm{IS}}(T)$ remain comparable, so the advantage of \eqref{eq:isgrpo_avg} over \eqref{eq:reinforce_avg} comes from the first-term factor $\tfrac45C(T)$; at the optimal step size the IS-GRPO guarantee is tighter throughout, by a factor that grows as prompts are solved. The heterogeneity across prompts that develops is dominated by the contrast between solved and active prompts, which per-prompt normalization cannot exploit and IS-GRPO's selection rule can, consistent with \Cref{fig:acc_7b}(b). \\

\vspace{-4ex}

\section{Conclusion}
\label{sec:con}

\vspace{-2ex}

GRPO's normalization implements a curvature-adaptive gradient mechanism. This perspective explains when normalization helps and why GRPO can outperform REINFORCE. Theoretically, we establish faster convergence guarantees for GRPO and IS-GRPO, a newly introduced algorithm, under mild assumptions. Empirically, experiments on GSM8K and MATH validate the curvature--variance connection, show that per-prompt normalization improves convergence when per-prompt variances are heterogeneous and brings no gain when they are not, and confirm that the conditions of our bounds $C(T)<1$ and a tighter guarantee for IS-GRPO, hold along real training runs.

\paragraph{Limitation.} Our analysis focuses on on-policy GRPO in the RLVR and log-linear policy parametrization. Extending the framework to more general RLHF settings, off-policy training, larger-scale models, and more challenging reasoning datasets remains an important direction. Our three-seed 7B experiments measure accuracy on the training prompts, the quantity our bounds concern; held-out evaluation at this scale is limited to a single-seed run.

\newpage

\subsection*{AI use statement}

In this work, we used generative AI tools to aid and polish the writing. We did not use generative AI tools to formulate or refine mathematical claims, draft derivations, to design the research methodology or experiments, to implement methods, or to interpret results. The remaining tasks with required disclosure, including generating synthetic datasets, translation, dataset cleaning and reformatting, and qualitative or thematic data analysis, are not applicable to this work. Additionally, we used generative AI tools for tasks with recommended disclosure: drafting and editing parts of the paper to improve readability, identifying and summarizing relevant literature, and formatting references. We have reviewed all AI-assisted work. We take responsibility for the final content of this work, including text, claims, or artifacts produced with the aid of generative AI.

\bibliography{reference}
\bibliographystyle{iclr2027_conference}

%%%%%%%%%%%%%%%%%%%%%%%%%%%%%%%%%%%%%%%%%%%%%%%%%%%%%%%%%%%%%%%%%%%%%%%%%%%%%%%
%%%%%%%%%%%%%%%%%%%%%%%%%%%%%%%%%%%%%%%%%%%%%%%%%%%%%%%%%%%%%%%%%%%%%%%%%%%%%%%
% APPENDIX
%%%%%%%%%%%%%%%%%%%%%%%%%%%%%%%%%%%%%%%%%%%%%%%%%%%%%%%%%%%%%%%%%%%%%%%%%%%%%%%
%%%%%%%%%%%%%%%%%%%%%%%%%%%%%%%%%%%%%%%%%%%%%%%%%%%%%%%%%%%%%%%%%%%%%%%%%%%%%%%
\newpage
\appendix

\section{Related Work}
\label{app:rel}
\paragraph{REINFORCE-style PG methods.} \citet{li2023remax} proposes ReMax, a sequence-level REINFORCE-style algorithm for LLM alignment. \citet{ahmadian2024back} propose RLOO, which samples multiple responses per prompt and uses a leave-one-out baseline to further reduce variance. \citet{li2026optimal} derives an optimal token-level baseline for variance reduction in long-horizon LLM RL, continuing this critic-free line of work.

\paragraph{GRPO and its variants.} GRPO \citep{shao2024deepseekmath} and related methods \citep{simoni2025gtpo, liu2025understanding, chu2025gpg, lin2025cppo} have gained widespread adoption due to their simplicity and scalability, and are now commonly used in post-training pipelines for reasoning-oriented models. Several extensions improve stability in long-CoT RL. DAPO \citep{yu2025dapo} introduces several techniques, including higher clipping thresholds and sequence-level losses, to make the training more stable. GSPO \citep{zheng2025group} defines the importance ratio at the sequence level from sequence likelihood, avoiding high-variance token-level weights. CISPO \citep{chen2025minimax} further revises the clipping strategy to avoid low-probability tokens from being clipped out after the first on-policy update. Constrained GRPO \citet{girgis2026constrained} extends GRPO to constrained settings and standardizes reward and cost advantages separately, so that normalization does not distort the intended trade-off.

\paragraph{Emerging theory for GRPO.} Recent studies analyze what GRPO optimizes and how it behaves in on- and off-policy regimes \citep{mroueh2025revisiting}, its implicit alignment objective \citep{vojnovic2025alignment}, and trajectory-corrected variants with convergence guarantees \citep{pang2025theory}. \citet{zhou2026demystifying} characterize the GRPO gradient as a U-statistic, which yields finite-sample error bounds and guidance on the group size.\citet{bereket2025uncalibrated} highlights a trade-off between normalization and calibration, showing that removing the std term can improve probability calibration at the cost of optimization speed. \citet{jaintowards} interprets GRPO's objective as a weighted combination of likelihood maximization for correct rollouts and likelihood minimization for incorrect ones, while noting that std normalization lacks a clear theoretical motivation. We contribute a new perspective, interpreting the std term as an adaptive gradient mechanism that reflects local curvature, which helps unify a range of previously disparate empirical findings.

\paragraph{RLVR.} Reinforcement learning with verifiable rewards (RLVR) has emerged as an effective paradigm for reasoning-intensive domains. Unlike RLHF, which relies on a learned reward model, RLVR uses deterministic, verifiable rewards \citep{lambert2024tulu,guo2025deepseek,team2025kimi,wang2025reinforcement}. This mitigates reward-model bias and hacking and simplifies training, while scaling effectively with compute and dataset size. Strong results have been reported on the well-recognized GSM8K \citep{cobbe2021gsm8k}, MATH \citep{hendrycks2021measuring}, Omni-MATH \citep{gao2024omni}, and FormalMATH \citep{yu2025formalmath}. Recent work extends the paradigm with verifiable process rewards for multi-step agentic reasoning \citep{yuan2026verifiable} and reference-based verifiable rewards for open-ended generation \citep{jiang2026verifiable}. In this paper, we focus on developing a theoretical understanding of GRPO in the RLVR setting. 

\begin{table}[ht]
\centering
\caption{List of notations.}
\label{tab:notation}
\begin{tabular}{ll}
\toprule
\textbf{Notations} & \textbf{Definition} \\
\midrule
$\mathcal{Q}$ & Set of questions / prompts \\
$\mathcal{O}$ & Set of possible output sequences \\
$o^\ast(q)$ & Unique correct output for prompt $q$ (\Cref{ass:ass_1}) \\
$a_i$ & Index of the correct output for prompt $q_i$ \\
${r}_i$ & Reward vector for prompt $q_i$ \\
$\pi_\theta$ & Sequence-level policy of the LLM parameterized by $\theta$ \\
$\pi_\theta(i)$ & Probability vector for prompt $q_i$ under policy $\pi_{\theta}$ \\
$\pi_\theta^\ast(i)$ & Success probability on prompt $q_i$ under policy $\pi_{\theta}$ \\
$J_i(\theta)$ & Objective function for prompt $q_i$ under policy $\pi_{\theta}$ \\
$J(\theta)$ & Global objective under policy $\pi_{\theta}$ \\
$J_{\mathrm{GRPO}}(\theta)$ & GRPO objective under policy $\pi_{\theta}$ \\
${x}_{i,j}$ & Feature vector for prompt-output pair $(q_i,o_j)$ \\
$X_i$ & Feature matrix for $q_i$ with rows ${x}_{i,1},\dots,{x}_{i,K}$ \\
$X_{\max}$ & Maximal feature norm, $X_{\max} = \max_{i\in[n]} \|X_i\|$ \\
$H(\pi_\theta(i))$ & Covariance matrix $\mathrm{diag}(\pi_\theta(i))-\pi_\theta(i)\pi_\theta(i)^\top$ \\
$\bV(\pi_\theta(i))$ & Variance of reward for prompt $q_i$ under policy $\pi_{\theta}$ \\
$\|\cdot\|$ & Euclidean norm (vectors) / spectral norm (matrices) \\
$\mathrm{diag}({v})$ & Diagonal matrix with vector ${v}$ on its diagonal \\
$\mathcal{B}({v},r)$ & Euclidean ball $\{{x}\in\mathbb{R}^m : \|{x}-{v}\|_2 \le r\}$ \\
$C_{i, t}$ & Prompt-wise variance bound introduced in \Cref{ass:ass_3} \\
\bottomrule
\end{tabular}
\end{table}

\section{Gradient Update Rule of REINFORCE and GRPO}
\label{app:grad_upd}
In this paper, we focus on the \emph{log-linear policy parametrization} \citep{agarwal2021theory, yuan2022linear}. Specifically, we assume that for each question-output pair $(q_i, o_j)$, there exists a constant feature vector ${x}_{i, j} \in \mathbb{R}^d$ and the policy is given by:
\begin{equation}
    \pi_{\theta}(o_j \mid q_i) := \frac{\exp({x}_{i, j}^{\top}\theta)}{\sum_{l = 1}^{K} \exp({x}_{i, l}^{\top} \theta)}.
\end{equation}
For ease of notation, we simply drop $t$ from $i_t$ whenever it clear in context.
The update of Algorithm~\ref{alg:reinforce} takes the following form \citep{lin2025rethinking}:
\begin{equation}
    \label{equ:update}
    \theta_{t} \leftarrow \theta_{t-1} + \eta X_i^{\top}\left(\diag\left(\pi_{\theta_{t-1}}(i)\right)-\pi_{\theta_{t-1}}(i) \pi_{\theta_{t-1}}^{\top}(i)\right){r}_i.
\end{equation}
Under \Cref{ass:ass_1}, the update of REINFORCE (\Cref{alg:reinforce}) can be simplified as:
\begin{equation}
\begin{aligned}
    \label{equ:update_1}
    \theta_{t} \leftarrow \theta_{t-1} + \eta \Big( \pi^*_{\theta_{t-1}}(i)(1 - \pi^*_{\theta_{t-1}}(i) {x}_{i,a_i}  
    - \pi^*_{\theta_{t-1}}(i) \sum_{j \neq a_i} [\pi_{\theta_{t-1}}(i)]_j \cdot {x}_{i,j}  \Big).
\end{aligned}
\end{equation}
Similarly, the update of GRPO (\Cref{alg:grpo}) can be simplified as:
\begin{equation}
\begin{aligned}
    \label{equ:update_2}
        \theta_{t} \leftarrow \theta_{t-1} + \eta \Big( \sqrt{\pi^*_{\theta_{t-1}}(i)(1 - \pi^*_{\theta_{t-1}}(i))} {x}_{i,a_i}  
        - \sqrt{\frac{\pi^*_{\theta_{t-1}}(i)}{1 - \pi^*_{\theta_{t-1}}(i)}} \sum_{j \neq a_i} [\pi_{\theta_{t-1}}(i)]_j \cdot {x}_{i,j} \Big).
\end{aligned}
\end{equation}

\section{Analysis of the Local Smoothness Constant}
\label{app:lemma}

\subsection{Proof of \Cref{lemma:local_smooth}} \label{appsub:lemma_1}

According to Lemma 17 in \cite{lin2025rethinking}, for any $y \in \mathbb{R}^d$, we have 
\begin{equation*}
    {y}^{\top} \nabla^2 J_{i}(\theta) {y} = \left(H\left(\pi_\theta(i)\right) {r}_i\right)^{\top}(X_i {y} \odot X_i {y})-2\left(H\left(\pi_\theta(i)\right) {r}_i\right)^{\top}(X_i {y})\left(\pi_\theta^{\top}(i) X_i {y}\right)
\end{equation*}
where $H(\pi_{\theta}(i))$ is defined as $H\left(\pi_\theta\right):=\diag\left(\pi_\theta(i)\right)-\pi_\theta(i) \pi_\theta^{\top}(i) \in \mathbb{R}^{K \times K}$ and $\odot$ denotes the Hadamard (component-wise) product. Using the triangle inequality and Cauchy-Schwarz inequality, we get
\begin{equation}
    \label{equ:norm_bound}
    \begin{aligned}
    |{y}^{\top} \nabla^2 J_{i}(\theta) {y}| &\le |\left(H\left(\pi_\theta(i)\right) {r}_i\right)^{\top}(X_i {y} \odot X_i {y})|+2|\left(H\left(\pi_\theta(i)\right) {r}_i\right)^{\top}(X_i {y})|\cdot|\left(\pi_\theta^{\top}(i) X_i {y}\right)| \\
    &\le \Vert \left(H\left(\pi_\theta(i)\right) {r}_i\right) \Vert_{\infty} \Vert X_i{y} \odot X_i{y}\Vert_1 + 2 \Vert H\left(\pi_\theta(i)\right) {r}_i\Vert \cdot \Vert X_i{y}\Vert \cdot \Vert\pi_{\theta}(i)\Vert \cdot \Vert X_i{y}\Vert \\
    &= \Vert \left(H\left(\pi_\theta(i)\right) {r}_i\right) \Vert_{\infty} \Vert X_i{y}\Vert^2 +  2 \Vert H\left(\pi_\theta(i)\right) {r}_i\Vert \cdot  \Vert\pi_{\theta}(i)\Vert \cdot \Vert X_i{y}\Vert^2 \\
    &\le \Vert \left(H\left(\pi_\theta(i)\right) {r}_i\right) \Vert_{\infty} \Vert X_i{y}\Vert^2 +  2 \Vert H\left(\pi_\theta(i)\right) {r}_i\Vert \cdot \Vert X_i{y}\Vert^2.
    \end{aligned}
\end{equation}
The last inequality follows because $\Vert \pi_{\theta}(i)\Vert \le \Vert \pi_{\theta}(i)\Vert_1 = 1$. According to \Cref{ass:ass_1}, we have
\begin{equation*}
    [H(\pi_{\theta}(i)){r}_i]_j = \begin{cases}
        \pi^*_{\theta}(i)(1 - \pi^*_{\theta}(i)), & \text{if } \ j = a_i \\
        -\pi^*_{\theta}(i) [\pi_{\theta}(i)]_j, & \text{if } \ j \neq a_i
    \end{cases}
\end{equation*}
With this expression, we get
\begin{equation}
    \label{equ:infty_bound}
    \begin{aligned}
        \Vert H(\pi_{\theta}(i)){r}_i \Vert_\infty &= \pi^*_{\theta}(i)(1 - \pi^*_{\theta}(i)),
    \end{aligned}
\end{equation}
and 
\begin{equation}
    \label{equ:2norm_bound}
    \begin{aligned}
        \Vert H(\pi_{\theta}(i)){r}_i \Vert &= \pi^*_{\theta}(i) \sqrt{(1 -\pi^*_{\theta}(i))^2 + \sum_{j \neq a_i}[\pi_{\theta}(i)]_j^2} \\
        &\le \pi^*_{\theta}(i) \sqrt{(1 -\pi^*_{\theta}(i))^2 + \sum_{j \neq a_i}[\pi_{\theta}(i)]_j(1 - \pi^*_{\theta}(i))} \\
        &= \sqrt{2} \pi^*_{\theta}(i)(1 - \pi^*_{\theta}(i)).
    \end{aligned}
\end{equation}
Combining (\ref{equ:infty_bound}) and (\ref{equ:2norm_bound}) with (\ref{equ:norm_bound}), we get
\begin{equation}
    \label{equ:final_bound}
    \begin{aligned}
        |{y}^{\top} \nabla^2 J_{i}(\theta) {y}| &\le \Vert \left(H\left(\pi_\theta(i)\right) {r}_i\right) \Vert_{\infty} \Vert X_i{y}\Vert^2 +  2 \Vert H\left(\pi_\theta(i)\right) {r}_i\Vert \cdot \Vert X_i{y}\Vert^2 \\
        &\le (2\sqrt{2} + 1)  \pi^*_{\theta}(i)(1 - \pi^*_{\theta}(i)) \Vert X_i {y}\Vert^2 \\
        &\le (2\sqrt{2} + 1)  \pi^*_{\theta}(i)(1 - \pi^*_{\theta}(i)) \Vert X_i \Vert^2 \Vert {y}\Vert^2 \\
        &\le 4 \pi^*_{\theta}(i)(1 - \pi^*_{\theta}(i)) X_{\max}^2 \Vert {y}\Vert^2
    \end{aligned}
\end{equation}
where the third inequality is due to the definition of operator norm, and the last inequality is by definition of $X_{\max}$. Note that 
\begin{equation*}
    \Vert \nabla^2 J_i(\theta) \Vert = \max_{{y}} \frac{|{y}^{\top} \nabla^2 J_{i}(\theta) {y}|}{\Vert {y}\Vert^2}
\end{equation*}
for symmetric Hessian matrix $\nabla^2 J_i(\theta)$, which completes the proof.

\subsection{Proof of \Cref{lemma:lipschitz}} \label{appsub:lemma_2}

According to~(\ref{equ:update_1}), the gradient of $J_{i}(\theta)$ takes the following form:
\begin{equation*}
    \nabla J_{i}(\theta) = {x}_{a_i}^{\top} (1 - \pi^*_{\theta_t}(i))\pi^*_{\theta_t}(i) - \sum_{j \neq a_i} {x}_j^{\top} \pi_{\theta_t}(i)_j \cdot \pi^*_{\theta_t}(i).
\end{equation*}
Note that a matrix's operator norm is larger than the norm of any of its row vector, we get
\begin{align*}
    \Vert \nabla J_{i}(\theta) \Vert &\le \Vert {x}_{a_i} \Vert (1 - \pi^*_{\theta_t}(i))\pi^*_{\theta_t}(i) + \sum_{j \neq a_i} \Vert {x}_j \Vert \pi_{\theta_t}(i)_j \cdot \pi^*_{\theta_t}(i) \\
    &\le \Vert X_i \Vert (1 - \pi^*_{\theta_t}(i))\pi^*_{\theta_t}(i) + \sum_{j \neq a_i} \Vert X_i \Vert \pi_{\theta_t}(i)_j \cdot \pi^*_{\theta_t}(i) \\
    &= 2 \Vert X_i \Vert (1 - \pi^*_{\theta_t}(i))\pi^*_{\theta_t}(i) \\
    &\le 2 \bV[\pi_{\theta}(i)] X_{\max},
\end{align*}
where the last inequality is due to the definition of $X_{\max}$, finishing the proof.

\subsection{Proof of \Cref{lemma:step_smooth}}
\label{appsub:lemma_3}

By \Cref{ass:ass_1}, the objective $J_i(\theta)$ is same as $\pi_{\theta}^*(i)$. From \Cref{lemma:lipschitz}, $J_i(\theta)$ is $\tfrac{1}{2}X_{\max}$-Lipschitz. Consequently, for any $\theta' \in \mathcal{B}\!\left(\theta, \tfrac{1}{X_{\max}} \sqrt{\pi_{\theta}^*(i)\big(1 - \pi_{\theta}^*(i)\big)}\right)$, we have
\begin{equation*}
    \big|\pi_{\theta'}^*(i) - \pi_{\theta}^*(i)\big| 
    \le \tfrac{1}{2}X_{\max} \cdot \tfrac{1}{X_{\max}} \cdot \sqrt{\pi_{\theta}^*(i)\big(1 - \pi_{\theta}^*(i)\big)} 
    = \tfrac{1}{2}\sqrt{\pi_{\theta}^*(i)\big(1 - \pi_{\theta}^*(i)\big)}.
\end{equation*}
Combining with \Cref{lemma:local_smooth}, 
\begin{equation*}
    \Vert \nabla^2J_i(\theta')\Vert \le \max_{l}\, 4X_{\max}^2 \cdot l(1-l)
\end{equation*}
over $\mathcal{B}\!\left(\theta, \tfrac{1}{X_{\max}} \sqrt{\pi_{\theta}^*(i)\big(1 - \pi_{\theta}^*(i)\big)}\right)$, where $l$ satisfies 
\[
\big|\,l - \pi_{\theta}^*(i)\,\big| \le \tfrac{1}{2}\sqrt{\pi_{\theta}^*(i)\big(1 - \pi_{\theta}^*(i)\big)}.
\]
We denote $\pi_{\theta}^*(i)$ as $a$. Thus, proving \Cref{lemma:step_smooth} is equivalent as proving 
\begin{equation*}
    f(a) \coloneqq \max_{l\in [a - \frac{\sqrt{a(1-a)}}{2}, a + \frac{\sqrt{a(1-a)}}{2}]} \frac{4l(1-l)}{\sqrt{a(1-a)}} \le \frac{5}{2}.
\end{equation*}
WLOG, we assume $a \in [0, \frac{1}{2}]$ and consider two cases.

\textbf{Case 1:} When $a \in [\frac{1}{2} - \frac{\sqrt{5}}{10}, \frac{1}{2}]$, we know that
\begin{equation*}
    \frac{1}{2} \in [a - \frac{\sqrt{a(1-a)}}{2}, a + \frac{\sqrt{a(1-a)}}{2}],
\end{equation*}
which implies that
\begin{equation*}
    f(a) = \frac{1}{\sqrt{a(1-a)}} \le f(\frac{1}{2} - \frac{\sqrt{5}}{10}) = \sqrt{5} \le \frac{5}{2}.
\end{equation*}
\textbf{Case 2:} When $a \in [0, \frac{1}{2} - \frac{\sqrt{5}}{10}]$, we know that
\begin{equation*}
    \frac{1}{2} \not \in [a - \frac{\sqrt{a(1-a)}}{2}, a + \frac{\sqrt{a(1-a)}}{2}].
\end{equation*}
Because $g(l) = l(1-l)$ is monotonically increasing in $[0, \frac{1}{2}]$, we can obtain  
\begin{equation*}
    f(a) = \frac{4(a + \frac{\sqrt{a(1-a)}}{2})(1 - a - \frac{\sqrt{a(1-a)}}{2})}{\sqrt{a(1-a)}} = 3 \sqrt{a(1-a)}  + (2 - 4a).
\end{equation*}
$f(a)$ takes its maximum when $a = \frac{1}{10}$ and $f(a) = \frac{5}{2}$. 

Combining the above two cases, we conclude the lemma.

\section{Convergence Analysis under Orthogonal Representation Assumption}
\label{app:thm}

\subsection{Auxiliary Lemma}
\begin{lemma}
    \label{lemma:auxiliary_1}
    Under \Cref{ass:ass_1} and~\ref{ass:ass_2}, for any $i, j \in [n], i \neq j$ and $\theta \in \mathbb{R}^d$, we have 
    \begin{equation}
         \nabla J_{i}(\theta)^{\top} \nabla J_{j}(\theta) = 0
    \end{equation}
\end{lemma}
\begin{proof}
    According to (\ref{equ:update}), we get
    \begin{align*}
        \nabla J_{i}(\theta)^{\top} \nabla J_{j}(\theta) &=  {r}_i^{\top}\left(\diag\left(\pi_{\theta}(i)\right)-\pi_{\theta}(i) \pi_{\theta}^{\top}(i)\right)X_iX_j^{\top}\left(\diag\left(\pi_{\theta}(j)\right)-\pi_{\theta}(j) \pi_{\theta}^{\top}(j)\right){r}_j \\
        &= {r}_i^{\top}\left(\diag\left(\pi_{\theta}(i)\right)-\pi_{\theta}(i) \pi_{\theta}^{\top}(i)\right){0}\left(\diag\left(\pi_{\theta}(j)\right)-\pi_{\theta}(j) \pi_{\theta}^{\top}(j)\right){r}_j \\
        &= 0,
    \end{align*}
    where the second step is by \Cref{ass:ass_2}.
\end{proof}

\subsection{Proof of \Cref{thm:reinforce}}
\label{appsub:thm_1}
We consider a specific question $q_l$. Combining \Cref{lemma:auxiliary_1} with \textit{log-linear policy parameterization} in our setting, if question $q_{i_t}$ is selected on iteration $t$ in \Cref{alg:reinforce}, we get
\begin{equation}
    \label{equ:seperate}
    \begin{aligned}
         J_j(\theta_{t}) &= J_j(\theta_{t-1} + \eta \nabla J_{i}(\theta_{t-1})) \\
         &= J_j(\theta_{t-1})
    \end{aligned}
\end{equation}
for any $i_t \neq l$. That is, the parameter update on question $q_{i_t}$ will not affect the expected reward on other questions. 

If question $i_t = l$ is selected on iteration $t$ in \Cref{alg:reinforce}, we have 
\begin{equation}
    \label{equ:single_reinforce}
    \begin{aligned}
        J_l(\theta_t) - J_l(\theta_{t-1}) &\ge \langle\theta_t -\theta_{t-1}, \nabla J_l(\theta_{t-1})\rangle - \frac{X_{\max}^2}{2} \Vert \theta_t - \theta_{t-1}\Vert^2 \\
        &= (\eta - \frac{X_{\max}^2}{2}\eta^2) \Vert \nabla J_{l}(\theta_{t-1})\Vert^2 \\
        &= \frac{1}{2X_{\max}^2} \Vert \nabla J_{l}(\theta_{t-1})\Vert^2
    \end{aligned}
\end{equation}
where the first step is by \Cref{lemma:local_smooth}, which also indicate that $J_i(\theta)$ is $X_{\max}^2$-weakly convex. Combining (\ref{equ:seperate}) and (\ref{equ:single_reinforce}) and taking the expectation over $i_t$, we get
\begin{equation}
    \label{equ:tele_reinforce}
    \begin{aligned}
        \mathbb{E}[J_l(\theta_t)] - \mathbb{E}[J_l(\theta_{t-1})] &\ge \frac{1}{2n X_{\max}^2} \mathbb{E}\Vert \nabla J_{l}(\theta_{t-1})\Vert^2.
    \end{aligned}
\end{equation} 
Summing up (\ref{equ:tele_reinforce}) for $t = 1, \dots, T$, we have
\begin{equation*}
    \frac{1}{2nX_{\max}^2} \sum_{t = 0}^{T-1} \mathbb{E}[\Vert \nabla J_l(\theta_{t-1})\Vert^2] \le \mathbb{E}[J_l(\theta_T)] - J_l(\theta_0) \le 1 - \pi^*_{\theta_0}(l).
\end{equation*}
This directly leads to 
\begin{equation*}
    \min_{t \in \{0, 1, \dots, T-1\}} \mathbb{E}[\Vert\nabla J_l(\theta_t)\Vert^2] \le \frac{2n(1 - \pi_{\theta_0}^*(l))X_{\max}^2}{T}.
\end{equation*}

\subsection{Proof of \Cref{thm:grpo}}
\label{appsub:thm_2}
Similar to (\ref{equ:seperate}) in the proof of \Cref{thm:reinforce}, the gradient update based on question $q_i$ does not affect the objective for question $q_l$ if $i \neq l$. That is,
\begin{equation}
    \label{equ:case}
    J_{l}(\theta_t) = \begin{cases}
        J_l(\theta_{t-1}), & \text{if $i_t \neq l$} \\
        J_l(\theta_{t}), & \text{if $i_t = l$}.
    \end{cases}
\end{equation}
Consider the case where $i_t = l$, from the parameter update rule in \Cref{alg:grpo}, we get
\begin{equation*}
    \theta_{t} = \theta_{t-1} + \eta \Big( \sqrt{\pi^*_{\theta_{t-1}}(l)(1 - \pi^*_{\theta_{t-1}}(l))} {x}_{l,a_l}  - \sqrt{\frac{\pi^*_{\theta_{t-1}}(l)}{1 - \pi^*_{\theta_{t-1}}(l)}} \sum_{j \neq a_l} [\pi_{\theta_{t-1}}(l)]_j \cdot {x}_{l,j} \Big).
\end{equation*}
Also, by setting $\eta = \frac{2}{5X_{\max}^2}$, we have 
\begin{align*}
    &\Vert \eta \Big( \sqrt{\pi^*_{\theta_{t-1}}(l)(1 - \pi^*_{\theta_{t-1}}(l))} {x}_{l,a_l}  - \sqrt{\frac{\pi^*_{\theta_{t-1}}(l)}{1 - \pi^*_{\theta_{t-1}}(l)}} \sum_{j \neq a_l} [\pi_{\theta_{t-1}}(l)]_j \cdot {x}_{l,j} \Big) \Vert \\
    = &\frac{2}{5 X_{\max}^2}\Vert \Big( \sqrt{\pi^*_{\theta_{t-1}}(l)(1 - \pi^*_{\theta_{t-1}}(l))} {x}_{l,a_l}  - \sqrt{\frac{\pi^*_{\theta_{t-1}}(l)}{1 - \pi^*_{\theta_{t-1}}(l)}} \sum_{j \neq a_l} [\pi_{\theta_{t-1}}(l)]_j \cdot {x}_{l,j} \Big) \Vert \\
    \le &\frac{1}{2 X_{\max}^2} \Big( \sqrt{\pi^*_{\theta_{t-1}}(l)(1 - \pi^*_{\theta_{t-1}}(l))} \Vert{x}_{l,a_l} \Vert + \sqrt{\frac{\pi^*_{\theta_{t-1}}(l)}{1 - \pi^*_{\theta_{t-1}}(l)}}\sum_{j \neq a_l} [\pi_{\theta_{t-1}}(l)]_j \cdot \Vert {x}_{l,j} \Vert\Big) \\
    \le & \frac{1}{2X_{\max}^2}(2\sqrt{\pi^*_{\theta_{t-1}}(l)(1 - \pi^*_{\theta_{t-1}}(l))} X_{\max}) \\
    = &\frac{1}{X_{\max}} \cdot \sqrt{\pi_{\theta_{t-1}}^*(l) (1 - \pi_{\theta_{t-1}}^*(l))}.
\end{align*}
This implies that $\theta_{t} \in \mathcal{B}(\theta, \frac{1}{X_{\max}} \cdot \sqrt{\pi_{\theta_{t-1}}^*(l) (1 - \pi_{\theta_{t-1}}^*(l))})$. According to \Cref{lemma:step_smooth}, we could obtain
\begin{equation}
    \label{equ:single_grpo}
    \begin{aligned}
        J_l(\theta_t) &\ge J_l(\theta_{t-1}) + \langle\theta_t -\theta_{t-1}, \nabla J_l(\theta_{t-1})\rangle - \frac{5}{4} X_{\max}^2 \cdot \sqrt{\pi_{\theta_{t-1}}^*(l) (1 - \pi_{\theta_{t-1}}^*(l))} \Vert \theta_t - \theta_{t-1}\Vert^2 \\
        &= J_l(\theta_{t-1}) + \frac{1}{5X_{\max}^2 \sqrt{\pi_{\theta}^*(l) (1 - \pi_{\theta}^*(l))}} \Vert \nabla J_l(\theta_{t-1})\Vert^2 \\
        &\ge J_l(\theta_{t-1}) + \frac{1}{5X_{\max}^2 C_{l, {t-1}}} \Vert \nabla J_l(\theta_{t-1})\Vert^2
    \end{aligned}    
\end{equation}
where the last step is by \Cref{ass:ass_3}. Combining \eqref{equ:case} and \eqref{equ:single_grpo} and taking the expectation over $i_t$, we obtain
\begin{equation}
    \label{equ:tele_grpo}
    \mathbb{E}[J_l(\theta_t)] \ge \mathbb{E}[J_l(\theta_{t-1})] + \frac{1}{5nX_{\max}^2 C_{l, {t-1}}} \mathbb{E}[\Vert \nabla J_l(\theta_{t-1})\Vert^2].
\end{equation}
Summing up (\ref{equ:tele_grpo}) for $t = 1, \dots, T$, we get
\begin{equation}
    \label{equ:one_grpo}
    \mathbb{E}[J_l(\theta_T)] \ge J_l(\theta_{0}) + \sum_{t = 0}^{T-1} \frac{1}{5nX_{\max}^2 C_{l, {t}}} \mathbb{E}[\Vert \nabla J_l(\theta_{t-1})\Vert^2].
\end{equation}
According to the Cauchy-Schwarz inequality, we have
\begin{equation*}
    \min_{t \in \{0, 1, \dots, T-1\}} \mathbb{E}[\Vert\nabla J_l(\theta_t)\Vert^2] \le \frac{2n(1 - \pi_{\theta_0}^*(l))X_{\max}^2}{T}  \frac{5 \sum_{t = 0}^{T-1} C_{l, {t}}}{2T}.
\end{equation*}

\section{Analysis of Importance Sampling GRPO}
\label{app:isgrpo}
This appendix proves the results in \Cref{subsec:isgrpo}.
 
\subsection{Auxiliary Lemma}
\begin{lemma}
    \label{lemma:var_stability}
    Under \Cref{ass:ass_1}, for all $i \in [n]$ and $\theta, \theta' \in \mathbb{R}^d$, with $\delta := X_{\max} \Vert \theta - \theta' \Vert$,
    \begin{equation*}
        e^{-2\delta}\, \mathbb{V}[\pi_{\theta}(i)] \;\le\; \mathbb{V}[\pi_{\theta'}(i)] \;\le\; e^{2\delta}\, \mathbb{V}[\pi_{\theta}(i)].
    \end{equation*}
\end{lemma}
\begin{proof}
According to the definition of $\bV[\pi_{\theta}(i)]$, we know
\begin{equation*}
    \bV[\pi_{\theta}(i)] = J_i(\theta)\bigl(1 - J_i(\theta)\bigr).
\end{equation*}
Since $J(\theta) \in (0,1)$, $\log \bV[\pi_{\theta}(i)]$ is differentiable with
\begin{equation*}
    \nabla \log \bV[\pi_{\theta}(i)] = \frac{\nabla \bV[\pi_{\theta}(i)]}{\bV[\pi_{\theta}(i)]} = \frac{\bigl(1 - 2J_i(\theta)\bigr)\, \nabla J_i(\theta)}{\bV[\pi_{\theta}(i)]}.
\end{equation*}
By \Cref{lemma:lipschitz} and $\abs{1 - 2J(\theta)} \le 1$, for all $\theta \in \mathbb{R}^d$,
\begin{equation}
    \label{eq:logV_grad}
    \Vert \nabla \log \bV[\pi_{\theta}(i)] \Vert
    \le \frac{\abs{1 - 2J(\theta)} \cdot 2\, \bV[\pi_{\theta}(i)]\, X_{\max}}{\bV[\pi_{\theta}(i)]}
    \le 2 X_{\max}.
\end{equation}
Thus, $\log \bV[\pi_{\theta}(i)]$ is $2 X_{\max}-$Lipschitz over $\mathbb{R}^d$, which directly leads to
\begin{equation*}
    \abs{\log \bV[\pi_{\theta'}(i)] - \log \bV[\pi_{\theta}(i)]}
    \le 2 X_{\max} \Vert \theta' - \theta \Vert
    = 2\delta.
\end{equation*}
Taking exponent on both sides gives $e^{-2\delta} \bV[\pi_{\theta}(i)] \le \bV[\pi_{\theta'}(i)] \le e^{2\delta} \bV[\pi_{\theta}(i)]$.
\end{proof}

\begin{lemma}
    \label{lemma:expect}
    Suppose \Cref{ass:ass_1} holds. Then, for every $t \ge 0$, we have
    \begin{equation*}
        \mathbb{E}\bigl[g_t\bigr] = \frac{n}{2 \sum_{i = 1}^n \sqrt{\bV[\pi_{\theta_t}(i)]}}\nabla J(\theta_{t}), \quad \Vert g_t \Vert \le X_{\max} / 2.
    \end{equation*}
\end{lemma}

\begin{proof}
We fix $t \ge 0$. As in the proof of \Cref{lemma:moments}, the index $i_{t}$ is drawn according to the positive probabilities $p_{t, i}$. For the first statement, we know
\begin{equation*}
    \mathbb{E}[g_t] = \sum_{i=1}^n p_{t, i}\, \frac{\nabla J_i(\theta_t)}{2\sqrt{\bV[\pi_{\theta_t}(i)]}} = \frac{n}{2 \sum_{i = 1}^n \sqrt{\bV[\pi_{\theta_t}(i)]}} \cdot \frac{1}{n} \sum_{i=1}^n \nabla J_i(\theta_t) = \frac{n}{2 \sum_{i = 1}^n \sqrt{\bV[\pi_{\theta_t}(i)]}}\nabla J(\theta_{t}).
\end{equation*}
For the second statement, the definition of $p_{t, i_{t}}$ gives
\begin{equation*}
    \Vert g_t\Vert
    = \frac{\Vert \nabla J_{i_{t}}(\theta_t)\Vert}{2\sqrt{\bV[\pi_{\theta_t}(i_{t})]}}.
\end{equation*}
By \Cref{lemma:lipschitz}, we know $\Vert \nabla J_{i_{t}}(\theta_t)\Vert \le 2X_{\max}\, \bV[\pi_{\theta_t}(i_{t})]$. Therefore,
\begin{equation*}
    \Vert g_t\Vert
    \le X_{\max}\, \sqrt{\bV[\pi_{\theta_t}(i_{t})]}
    \le \frac{X_{\max}}{2},
\end{equation*}
which completes the proof.
\end{proof}
 
\begin{lemma}
    \label{lemma:moments}
    Suppose \Cref{ass:ass_1} holds. Then, for every $t \ge 0$, we have
    \begin{equation*}
        \mathbb{E}\bigl[\Vert g_t \Vert^2 \bigr]
        \;\le\; X_{\max}^2\, \frac{\sum_{i = 1}^n \bV[\pi_{\theta_t}(i)]^{\frac{3}{2}}}{\sum_{i = 1}^n \bV[\pi_{\theta_t}(i)]^{\frac{1}{2}}}.
    \end{equation*}
\end{lemma}
\begin{proof}
    We fix $t \ge 0$. According to \Cref{alg:isgrpo}, the index $i_{t}$ is drawn according to the probabilities
    \begin{equation*}
        p_{t, i} = \frac{\sqrt{\bV[\pi_{\theta_{t}}(i)]}}{\sum_{j=1}^n \sqrt{\bV[\pi_{\theta_{t}}(j)]}}
    \end{equation*}
    for $i \in [n]$. Since $\bV[\pi_{\theta_t}(i)] \in (0, \frac{1}{4}]$, so every $p_{t, i} > 0$. Consequently,
    \begin{equation*}
        \mathbb{E}\bigl[\Vert g_t \Vert^2\bigr]
        = \sum_{i=1}^n p_{t, i}\, \frac{\Vert \nabla J_i(\theta_t)\Vert^2}{4\bV[\pi_{\theta_t}(i)]}
        = \frac{1}{\sum_{j=1}^n \sqrt{\bV[\pi_{\theta_t}(j)]}}
        \cdot \sum_{i=1}^n \frac{\Vert \nabla J_i(\theta_t)\Vert^2}{4\sqrt{\bV[\pi_{\theta_t}(i)]}}.
    \end{equation*}
    According to \Cref{lemma:lipschitz}, we know $\Vert \nabla J_i(\theta_t)\Vert^2 \le 4X_{\max}^2\, \bV[\pi_{\theta_t}(i)]^2$, which directly leads to
    \begin{equation*}
        \frac{\Vert \nabla J_i(\theta_t)\Vert^2}{\sqrt{\bV[\pi_{\theta_t}(i)]}}
        \le 4X_{\max}^2\, \bV[\pi_{\theta_t}(i)]^{3/2}.
    \end{equation*}
    Thus, we obtain
    \begin{equation*}
        \mathbb{E}\bigl[\Vert g_t \Vert^2\bigr]
        \le X_{\max}^2\, \frac{\sum_{i = 1}^n \bV[\pi_{\theta_t}(i)]^{\frac{3}{2}}}{\sum_{i = 1}^n \bV[\pi_{\theta_t}(i)]^{\frac{1}{2}}},
    \end{equation*}
    which completes the proof.
\end{proof}

\begin{lemma}
    \label{lemma:is_comparison}
    For any fixed $t \ge 0$ and positive constants $\{l_i\}_{i = 1}^n$, we have
    \begin{equation*}
        \big(\sum_{i} l_i^2\big) \big(\sum_{i} l_i^3\big) \;\le\; \big(\sum_{i} l_i\big) \big(\sum_{i} l_i^4\big),
    \end{equation*}
    where the inequality is strict unless $l_{1}=\cdots=l_{n}$.
\end{lemma}

\begin{proof}
We know
\begin{align*}
    \big(\sum_{i} l_i\big) \big(\sum_{i} l_i^4\big) - \big(\sum_{i} l_i^2\big) \big(\sum_{i} l_i^3\big)
    &= \frac{1}{2}\sum_{i=1}^n\sum_{j=1}^n l_{i}l_{j}\bigl(l_{i} - l_{j}\bigr)\bigl(l_{i}^2 - l_{j}^2\bigr) \\
    &= \frac{1}{2}\sum_{i=1}^n\sum_{j=1}^n l_{i}l_{j}\bigl(l_{i} - l_{j}\bigr)^2\bigl(l_{i} + l_{j}\bigr) \\
    &\ge 0,
\end{align*}
where the last step is because $l_{i}l_{j}\bigl(l_{i} - l_{j}\bigr)^2\bigl(l_{i} + l_{j}\bigr)$ is nonnegative, and it is strictly positive whenever $l_{i} \neq l_{j}$. Hence 
\begin{equation*}
    \big(\sum_{i} l_i^2\big) \big(\sum_{i} l_i^3\big) \;\le\; \big(\sum_{i} l_i\big) \big(\sum_{i} l_i^4\big),
\end{equation*}
and the inequality is strict unless $l_{1} = \cdots = l_{n}$, which completes the proof.
\end{proof}

\subsection{\texorpdfstring{Proof of \Cref{thm:reinforce_is}}{Proof of the REINFORCE Convergence Theorem}}
\label{appsub:reinforce_is}
We fix $t \ge 0$. In \Cref{alg:reinforce}, the index $i_t$ is drawn uniformly from $[n]$ and 
\begin{equation*}
    \theta_t = \theta_{t-1} + \eta\, \nabla J_{i_t}(\theta_{t-1}).
\end{equation*}
Acoording to \Cref{lemma:lipschitz} and $\sqrt{\bV[\pi_{\theta_{t}}(i)]} \le \frac{1}{2}$, the update is bounded by
\begin{equation*}
    \Vert \nabla J_{i_t}(\theta_{t})\Vert \le 2X_{\max}\, \bV[\pi_{\theta_{t}}(i_t)] \le \frac{X_{\max}}{2}.
\end{equation*}
Combining \Cref{lemma:lipschitz} with \Cref{ass:ass_3}, the expected second moment of update satisfies
\begin{equation}
    \label{eq:reinforce_moment}
    \mathbb{E}\bigl[\Vert \nabla J_{i_t}(\theta_{t})\Vert^2\bigr]
    = \frac{1}{n}\sum_{i=1}^n \Vert \nabla J_i(\theta_{t})\Vert^2
    \le \frac{4X_{\max}^2}{n}\sum_{i=1}^n \bV[\pi_{\theta_{t}}(i)]^2
    \le 4X_{\max}^2\, \mu_{4, t}.
\end{equation}
Since $0 < \eta \le \frac{1}{2X_{\max}^2}$, we know 
\begin{equation*}
    \Vert \theta_t - \theta_{t-1}\Vert = \eta \Vert \nabla J_{i_t}(\theta_{t-1})\Vert \le \eta X_{\max}/2 \le \frac{1}{4X_{\max}}.
\end{equation*}
For every $\theta'$ on the segment $[\theta_{t}, \theta_{t+1}]$ satisfies
\begin{equation*}
    \delta := X_{\max}\Vert \theta' - \theta_{t-1}\Vert \le \eta X_{\max}^2/2 \le \frac{1}{4}.
\end{equation*}
Combining with \Cref{lemma:local_smooth} and \Cref{lemma:var_stability}, we know
\begin{align*}
    \Vert \nabla^2 J(\theta')\Vert
    &\le \frac{1}{n}\sum_{i=1}^n \Vert \nabla^2 J_i(\theta')\Vert \\
    &\le \frac{4X_{\max}^2}{n}\sum_{i=1}^n \bV[\pi_{\theta'}(i)] \\
    &\le \frac{4e^{2\delta}X_{\max}^2}{n}\sum_{i=1}^n \bV[\pi_{\theta_{t}}(i)] \\
    &\le 8X_{\max}^2\, \mu_{2, t}.
\end{align*}
Applying Taylor's theorem, we could obtain
\begin{equation*}
    J(\theta_{t+1}) \ge J(\theta_{t})
    + \eta\,\langle \nabla J(\theta_{t}), \nabla J_{i_t}(\theta_{t})\rangle
    - 4X_{\max}^2\mu_{2, t}\,\eta^2\,\Vert \nabla J_{i_t}(\theta_{t})\Vert^2.
\end{equation*}
Combining \eqref{eq:reinforce_moment} and taking expectation, we have 
\begin{equation*}
    \mathbb{E}[J(\theta_{t+1})]
    \ge J(\theta_{t}) + \eta\,\Vert \nabla J(\theta_{t})\Vert^2
    - 16\eta^2 X_{\max}^4\, \mu_{2, t}\,\mu_{4, t}.
\end{equation*}
Summing over $t = 1, \dots, T$, and using $J \le 1$ together with the definition of $D(T)$,
\begin{equation*}
    \eta \sum_{t=0}^{T-1} \mathbb{E}\Vert \nabla J(\theta_t)\Vert^2
    \le 1 - J(\theta_0) + 8\eta^2 X_{\max}^4\, T\, D(T).
\end{equation*}
Dividing by $\eta T$ gives \eqref{eq:reinforce_avg}, which completes the proof.
 
\subsection{\texorpdfstring{Proof of \Cref{thm:isgrpo}}{Proof of the IS-GRPO Convergence Theorem}}
\label{appsub:isgrpo_thm}
We fix $t \ge 0$. In \Cref{alg:isgrpo}, the update is 
\begin{equation*}
    \theta_{t+1} = \theta_t + \eta\, g_t.
\end{equation*}
According to \Cref{lemma:expect}, $\Vert \theta_{t+1} - \theta_t\Vert = \eta \Vert g_t\Vert \le \eta X_{\max}/2 \le \frac{1}{4X_{\max}}$, so every $\theta'$ on the segment $[\theta_t, \theta_{t+1}]$ again satisfies
\begin{equation*}
    \delta := X_{\max}\Vert \theta' - \theta_t\Vert \le \frac{1}{4}.
\end{equation*}
Combining with \Cref{lemma:local_smooth} and \Cref{lemma:var_stability}, we know
\begin{align*}
    \Vert \nabla^2 J(\theta')\Vert
    &\le \frac{1}{n}\sum_{i=1}^n \Vert \nabla^2 J_i(\theta')\Vert \\
    &\le \frac{4X_{\max}^2}{n}\sum_{i=1}^n \bV[\pi_{\theta'}(i)] \\
    &\le \frac{4e^{2\delta}X_{\max}^2}{n}\sum_{i=1}^n \bV[\pi_{\theta_{t}}(i)] \\
    &\le 8X_{\max}^2\, \frac{\sum_{i=1}^n \bV[\pi_{\theta_{t}}(i)]}{n}.
\end{align*} 
Applying Taylor's theorem, we obtain
\begin{equation*}
    J(\theta_{t+1}) \ge J(\theta_t)
    + \eta\,\langle \nabla J(\theta_t), g_t\rangle
    - 4X_{\max}^2 \frac{\sum_{i=1}^n \bV[\pi_{\theta_{t}}(i)]}{n} \cdot \eta^2\,\Vert g_t\Vert^2.
\end{equation*}
Combining with \Cref{lemma:expect,lemma:moments}, we could obtain
\begin{align*}
    \mathbb{E}[J(\theta_{t+1})]
    &\ge J(\theta_t) + \frac{n \cdot \eta}{2 \sum_{i = 1}^n \sqrt{\bV[\pi_{\theta_t}(i)]}}\,\Vert \nabla J(\theta_t)\Vert^2
    - 4\eta^2 X_{\max}^4 \frac{(\sum_{i = 1}^n \bV[\pi_{\theta_t}(i)]^{\frac{3}{2}})(\sum_{i = 1}^n \bV[\pi_{\theta_t}(i)])}{n(\sum_{i = 1}^n \bV[\pi_{\theta_t}(i)]^{\frac{1}{2}})} \\
    &\ge J(\theta_t) + \frac{n \cdot \eta}{2 \sum_{i = 1}^n \sqrt{\bV[\pi_{\theta_t}(i)]}}\,\Vert \nabla J(\theta_t)\Vert^2 - 4 \eta^2 X_{\max}^4 \frac{\sum_{i=1}^n \bV[\pi_{\theta_{t}}(i)]^2}{n}
\end{align*}
We sum over $t = 0, \dots, T-1$ and obtain
\begin{equation*}
    \sum_{t=0}^{T-1} \frac{n \cdot \eta}{2 \sum_{i = 1}^n \sqrt{\bV[\pi_{\theta_t}(i)]}} \mathbb{E}\Vert \nabla J(\theta_t)\Vert^2
    \le 1 - J(\theta_0) + 4\eta^2 X_{\max}^4\,\, \sum_{t = 0}^{T-1}\mu_{4, t}.
\end{equation*}
Applying $\frac{n \cdot \eta}{2 \sum_{i = 1}^n \sqrt{\bV[\pi_{\theta_t}(i)]}} \ge \frac{\eta}{2\mu_{1, t}}$, we obtain \eqref{eq:isgrpo_avg}, which completes the proof.

\section{Additional Experimental Results}
\label{app:additional_results}

%% ---------------------------------------------------------------------
\subsection{Experimental Setup Details}\label{app:setup_details}
 
\paragraph{Normalization strategies.}
All schemes use the per-prompt mean as baseline and differ in the scaling factor:
\begin{itemize}[leftmargin=*,itemsep=1pt]
    \item \textbf{Per-prompt std} (standard GRPO): $\hat{A}_{i}(o) = \frac{r_i - \text{mean}_i(r)}{\text{std}_i(r)}$, with mean and std computed within each prompt's group of $G$ rollouts.
    \item \textbf{Global std}: $\hat{A}_{i}(o) = \frac{r_i - \text{mean}_{i}(r)}{\text{std}_{\text{global}}(r)}$, with the std computed over all rollouts of the mini-batch and shared across prompts. Both per-prompt and global std rescale by a standard deviation; only the former adapts the factor to each prompt.
    \item \textbf{REINFORCE} (unnormalized): $\hat{A}_{i}(o) = r_i - \text{mean}_i(r)$.
\end{itemize}
\paragraph{IS-GRPO in practice.}
\Cref{alg:isgrpo} samples prompt $i$ with probability $\propto\sqrt{\bV[\pi_{\theta_t}(i)]}$ and updates it with $\nabla J_i/(2\sqrt{\bV[\pi_{\theta_t}(i)]})$. The true standard deviations are unknown, and estimating them for all prompts at every iteration would require rolling out the whole training set. We therefore keep a tracker $s_i$ for each of the training prompts:
\begin{itemize}[leftmargin=*,itemsep=1pt]
  \item \emph{Initialization.} $s_i=1/2$, the largest possible value, so that every prompt is likely to be drawn at least once.
  \item \emph{Update.} Whenever prompt $i$ is rolled out, $s_i\leftarrow0.9\,s_i+0.1\sqrt{\hat p_i(1-\hat p_i)}$, with $\hat p_i$ the fraction of correct rollouts; the moving average smooths the noise of an estimate from $G=8$ samples.
\end{itemize}
Each batch of 16 prompts is drawn with probability $\propto s_i$, so prompts that are nearly always solved or nearly always failed, which carry little gradient, are rarely drawn. Each drawn prompt's advantage is divided by its within-group std, which implements the factor $1/\sqrt{\bV}$ of \Cref{alg:isgrpo} exactly as in GRPO; the constant $1/2$ is absorbed into the learning rate. Since $s_i$ is a noisy and lagged estimate of the true standard deviation, this is an approximation of \Cref{alg:isgrpo} (see also \Cref{app:coef_details}).
 
\paragraph{Model scales.}
\begin{itemize}[leftmargin=*,itemsep=1pt]
    \item \textbf{1.5B:} Qwen2.5-Math-1.5B with LoRA on all linear layers ($r=64$, $\alpha=128$), 3 seeds, trained on GSM8K.
    \item \textbf{7B:} Qwen2.5-Math-7B with LoRA on all linear layers ($r=32$, $\alpha=32$), in two runs: MATH Level~3--5 with held-out evaluation on AMC (1 seed, 400 iterations), and MATH Level~1--5 with evaluation on the training prompts (3 seeds, 500 iterations, including IS-GRPO); see \Cref{app:7b_details}.
\end{itemize}
 
\paragraph{Learning rate control.}
Each strategy was swept over $\eta\in\{5\!\times\!10^{-7},\,1\!\times\!10^{-6},\,5\!\times\!10^{-6},\,1\!\times\!10^{-5},\,5\!\times\!10^{-5}\}$ at 1.5B scale. All methods are reported at the REINFORCE baseline's optimal learning rate.
 
\subsection{Validation of the Orthogonal Assumption}
\label{app:f2}

\Cref{ass:ass_2} posits that representations of different prompts are nearly orthogonal. We check this at two levels with Qwen2.5-Math-1.5B on GSM8K: the gradient level during training is presented in \Cref{subsec:curvature_variance}; here we add a representation-level check at initialization.

\paragraph{Representation level.}
We sample 1{,}000 question pairs, extract penultimate-layer hidden states and compute absolute cosine similarities. Similarities concentrate near zero (mean\,$\approx$\,0.088, std\,$\approx$\,0.064), and over 90\% of pairs fall below~0.15 (\Cref{fig:cosine_analysis}).

\paragraph{Relation to the theory.}
Gradient orthogonality is a necessary consequence of \Cref{ass:ass_2} under the log-linear parameterization, so these statistics test a necessary condition. Together they support \Cref{ass:ass_2} as an accurate description early in training and document its gradual erosion. \Cref{thm:grpo} relies on exact orthogonality; \Cref{thm:reinforce_is,thm:isgrpo} do not. We use the spread of the cosine distribution as a proxy for cross-prompt interference; it marks the transition from Phase~II to Phase~III in \Cref{subsec:4.5}; its growth is consistent with the 1.5B gaps no longer widening late in training.

\begin{figure}[h]
\centering
\begin{minipage}[b]{0.42\linewidth}
  \centering
  \includegraphics[width=\linewidth]{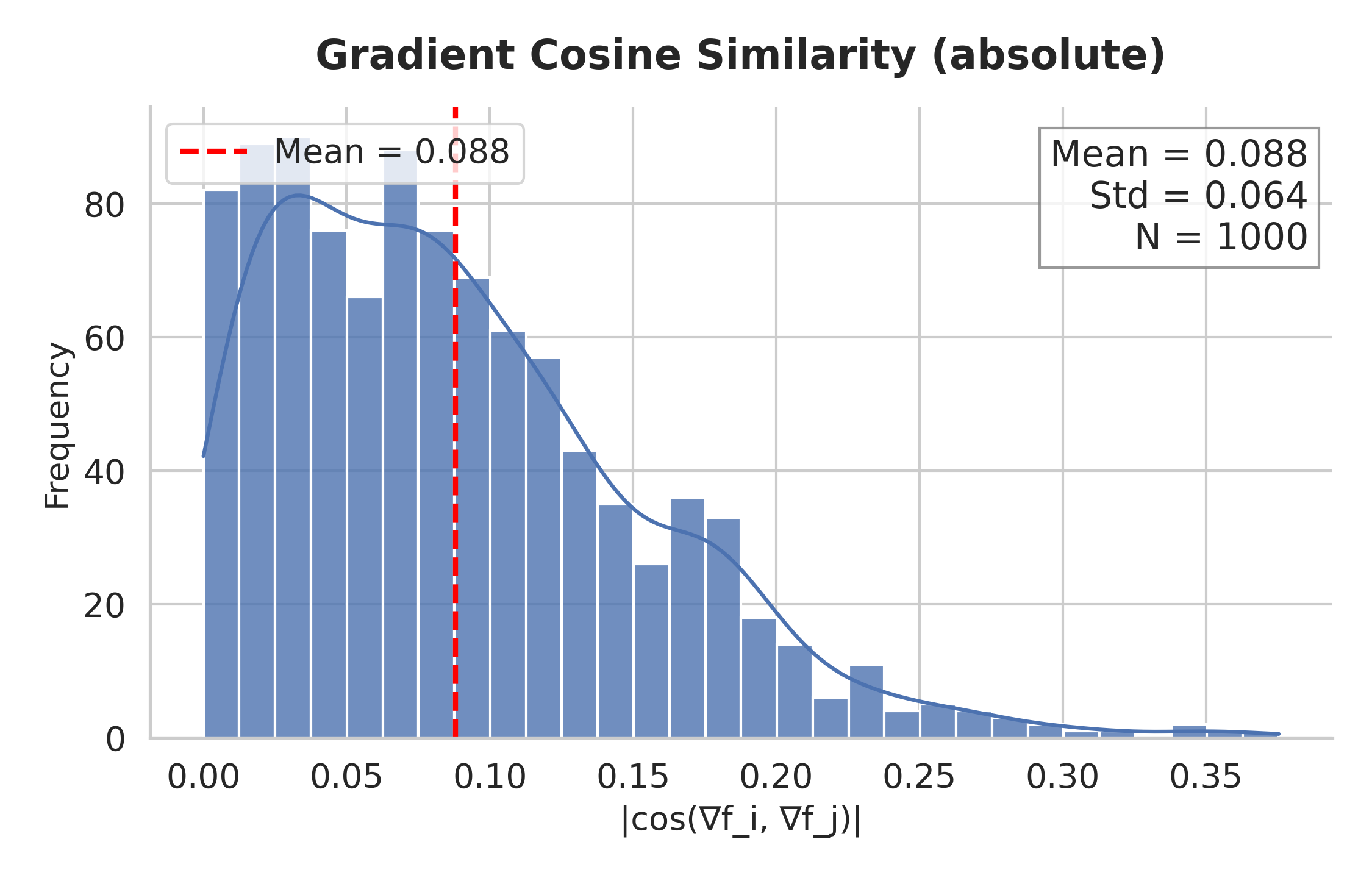}\\[-2pt]
  {\scriptsize (a) Histogram of $|\!\cos\!|$}
\end{minipage}\hfill
\begin{minipage}[b]{0.42\linewidth}
  \centering
  \includegraphics[width=\linewidth]{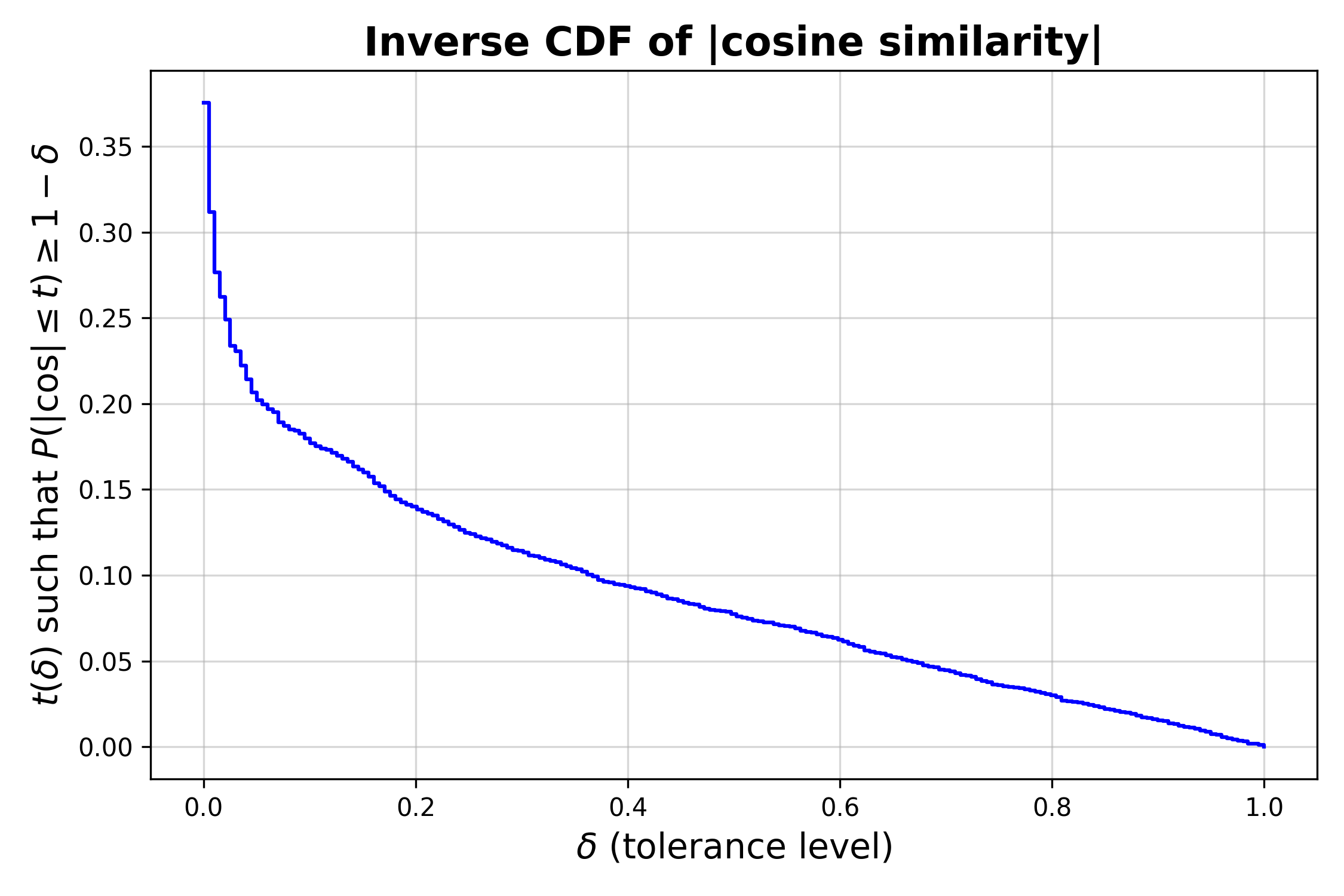}\\[-2pt]
  {\scriptsize (b) Inverse CDF of $|\!\cos\!|$}
\end{minipage}
\caption{Validation of \Cref{ass:ass_2} (1.5B, GSM8K, 1{,}000 question pairs).}
\label{fig:cosine_analysis}
\end{figure}

%% ---------------------------------------------------------------------
\subsection{Fisher Information Estimation Details}\label{app:fisher_method}
 
Given a mini-batch of prompts $\{q_i\}_{i=1}^B$ at iteration $t$, we estimate the diagonal Fisher information as follows:
\begin{enumerate}[leftmargin=*,itemsep=1pt]
    \item Sample responses $\hat{o}_i \sim \pi_{\theta_t}(\cdot | q_i)$ for each prompt.
    \item Compute the mini-batch log-likelihood gradient $\nabla \hat{\mathcal{L}}_B(\theta_t) = \frac{1}{B} \sum_{i=1}^B \nabla \log \pi_{\theta_t}(\hat{o}_i | q_i)$.
    \item Estimate the diagonal Fisher via the outer-product estimator $h(\theta_t) := B \cdot \nabla \hat{\mathcal{L}}_B(\theta_t) \odot \nabla \hat{\mathcal{L}}_B(\theta_t)$, where $\odot$ is element-wise multiplication; $\mathbb{E}_{\hat{o}}[h(\theta_t)] = \mathbb{E}_{\hat{o}}[\mathrm{diag}(\hat{F}(\theta_t))]$.
\end{enumerate}
The per-prompt curvature proxy is $\|h_i(\theta_t)\|_1$, the $\ell_1$ norm of the prompt-level diagonal Fisher entries. The Fisher information of the log-likelihood is related to, but not identical with, the Hessian of $J_i$ bounded in \Cref{lemma:local_smooth}; we use it because it is cheap to estimate at scale. \Cref{tab:fisher_correlation} correlates this proxy with the per-prompt reward variance: the association is positive and significant when both are measured at the same checkpoint, and absent when they are measured at different checkpoints, as expected if variance tracks the curvature at the current parameters.
 
%% ---------------------------------------------------------------------

\subsection{Detailed Phase-by-Phase Analysis (1.5B)}\label{app:phase_detail}
This section expands on \Cref{subsec:4.5}. \Cref{tab:final_accuracy_full} gives the final accuracies with across-seed standard deviations.

\begin{table}[h]
  \centering\small
  \caption{1.5B: final accuracy (\%, mean $\pm$ std over 3 seeds) at iteration 500.}
  \label{tab:final_accuracy_full}
  \setlength{\tabcolsep}{5pt}
  \begin{tabular}{lccc}
    \toprule
    Method & GSM8K-Easy & GSM8K-Hard & MATH Lv.\,2 \\
    \midrule
    Per-prompt std & $\mathbf{97.2\!\pm\!0.4}$ & $\mathbf{87.1\!\pm\!1.2}$ & $\mathbf{58.3\!\pm\!0.9}$ \\
    Global std     & $96.1\!\pm\!0.5$ & $85.3\!\pm\!1.3$ & $55.7\!\pm\!1.1$ \\
    REINFORCE      & $95.1\!\pm\!0.6$ & $84.2\!\pm\!1.5$ & $54.1\!\pm\!1.3$ \\
    \bottomrule
  \end{tabular}
\end{table}
 
\paragraph{Convergence factor $C(T)$.}
\Cref{fig:gsm8k_1.5b_ct} tracks $c_t=\frac{5}{2n}\sum_i\sqrt{\pi^*_{\theta_t}(i)(1-\pi^*_{\theta_t}(i))}$ and its running average $C(T)$. On both splits $C(T)$ stays below~1 throughout training, so the speed-up condition~\eqref{equ:normalization} of \Cref{thm:grpo} holds, and it ends at $0.28$ on Easy and $0.45$ on Hard. It falls faster on Easy, where more prompts saturate early. On Hard, $c_t$ first rises (from 0.42 to 0.56) as prompts leave the all-incorrect regime and become active, then falls as they are solved. Since $C(T)<1$ already holds in Phase~I, where the methods still overlap, the condition is necessary rather than sufficient for a visible gap: the gap additionally requires the scaling factors of active prompts to differ.
 
\begin{figure}[h]
\centering
\includegraphics[width=0.75\linewidth]{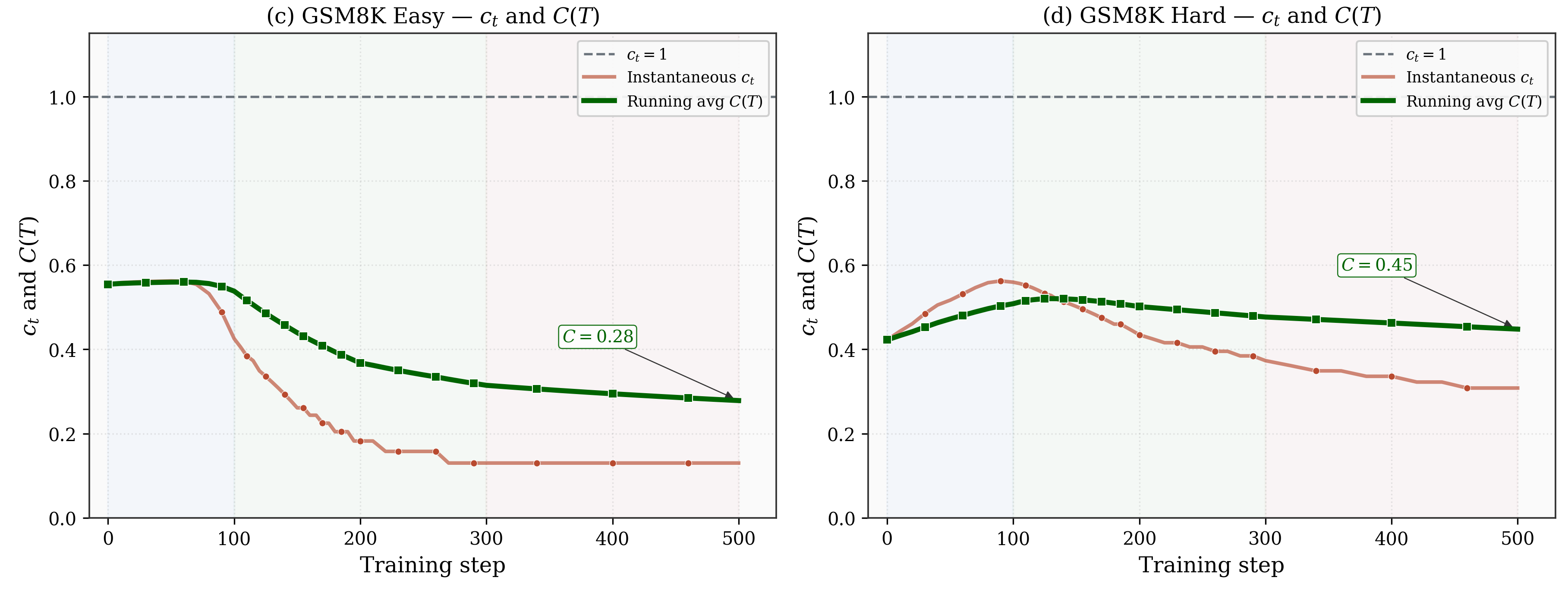}
\vspace{-4pt}
\caption{$c_t$ and $C(T)$ on GSM8K Easy and Hard (1.5B). $C(T)$ stays below~1 throughout training, ending at 0.28 (Easy) and 0.45 (Hard).}
\label{fig:gsm8k_1.5b_ct}
\end{figure}
 
\paragraph{Phase~I (iterations 0--100).}
All three schemes overlap. On GSM8K-Easy all methods improve from $\approx0.45$ to $\approx0.70$; on GSM8K-Hard from $\approx0.20$ to $\approx0.40$. With $G=8$ rollouts a large fraction of groups are all-correct or all-incorrect and receive zero advantage under every scheme; the prompts that do receive a gradient have similar success probabilities, so their scaling factors are nearly uniform and per-prompt and global normalization coincide.
 
\paragraph{Phase~II (iterations 100--300).}
As accuracies spread out, $c_t$ falls and per-prompt std establishes a widening advantage. On GSM8K-Easy it reaches $\approx0.95$ vs.\ $\approx0.93$ (global std) and $\approx0.92$ (REINFORCE); on GSM8K-Hard the separation is larger, $\approx0.85$ vs.\ $\approx0.80$ and $\approx0.78$. Pairwise gradient cosines remain small in this phase (\Cref{subsec:curvature_variance}).
 
\paragraph{Phase~III (iterations 300--500).}
Per-prompt std keeps its advantage but the gap no longer widens. On GSM8K-Easy all methods stabilize near 0.96--0.98 with a $\sim$2\,pp lead; on GSM8K-Hard they fluctuate around 0.80--0.88 with a $\sim$3\,pp lead. The growing spread of gradient cosines (std $\approx0.130$) coincides with the end of the widening; on Easy, the approach to the accuracy ceiling is a simpler explanation for it.
\paragraph{Final gaps.}
Per-prompt std outperforms global std by 1.1\,pp (Easy), 1.8\,pp (Hard) and 2.6\,pp (MATH\,Lv.\,2), and REINFORCE by 2.1\,pp, 2.9\,pp and 4.2\,pp respectively (\Cref{tab:final_accuracy_full}).

%% ---------------------------------------------------------------------
\subsection{7B Experiments: Details}\label{app:7b_details}
 
\paragraph{Held-out run (MATH, AMC).}
One seed per scheme, 400 iterations, learning rate $5\times10^{-6}$; AMC accuracy is measured with greedy decoding every 10 iterations. With a single seed and a small benchmark, individual checkpoints are noisy (\Cref{fig:acc_7b}a); we rely on the ordering that persists over the last 150 iterations rather than on any single gap. At iteration 400, per-prompt std leads global std by $\approx$\,2\,pp and REINFORCE by $\approx$\,6\,pp.

\paragraph{Second held-out benchmark (AIME).}
On the same MATH Level~3--5 run we additionally evaluate on AIME 2024 (\Cref{fig:aime}), a harder, out-of-distribution benchmark. With a single seed and only 30 AIME problems, we do not read a remarkable result but at least we can see some trends shows a good tendency by using normalization from this run.

\begin{figure}[h]
\label{fig:aime}
\centering
\includegraphics[width=0.62\linewidth]{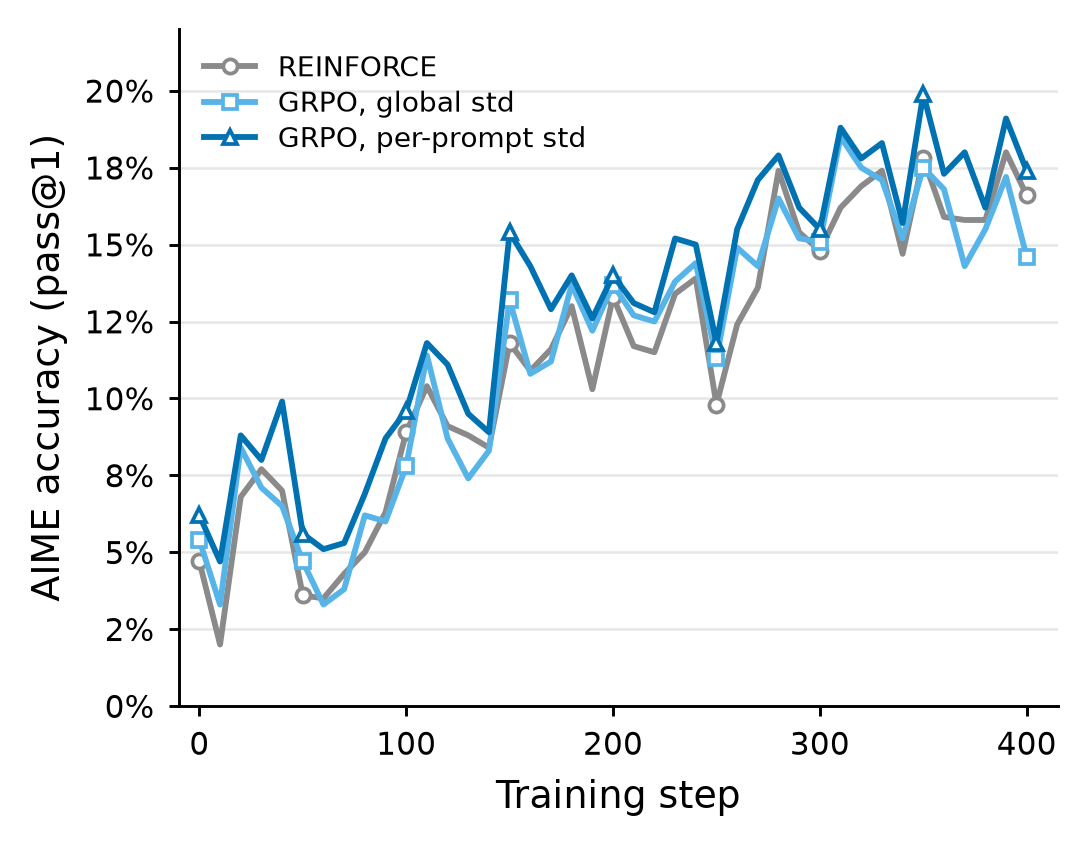}
\caption{7B held-out accuracy (pass@1) on AIME 2024, MATH Level~3--5 training run, one seed.}
\end{figure}
 
\paragraph{Training-prompt run (MATH Level~1--5).}
Three seeds per method, 500 iterations, 16 prompts and $G=8$ rollouts per iteration; greedy accuracy on the 256 training prompts at iterations 0, 50, 100, 200, 300, 400, 450 and 500. The model solves 68\% of the training prompts at iteration~0. \Cref{tab:acc_7b} reports the mean of the last three evaluations, whose across-seed standard deviation is 0.6\,pp, against 1.8\,pp for a single late checkpoint and 2.3\,pp for the early ones.
 
\begin{table}[h]
  \centering\small
  \caption{7B, MATH Level~1--5: greedy accuracy (\%) on the training prompts, 3 seeds.
  Late window: mean over iterations 400/450/500 ($\pm$ across-seed std).
  Step 500: single checkpoint (across-seed std $\approx$ 1.8\,pp).}
  \label{tab:acc_7b}
  \setlength{\tabcolsep}{6pt}
  \begin{tabular}{lcc}
    \toprule
    Method & Late window & Step 500 \\
    \midrule
    IS-GRPO        & $\mathbf{87.1\!\pm\!0.6}$ & 88.7 \\
    REINFORCE      & $86.1\!\pm\!0.6$ & 86.3 \\
    Per-prompt std & $85.9\!\pm\!0.6$ & 87.1 \\
    Global std     & $85.4\!\pm\!0.6$ & 85.9 \\
    \bottomrule
  \end{tabular}
\end{table}

 \paragraph{Why the normalization schemes coincide in the training-prompt run.}
Three features of this run reduce the leverage of per-prompt normalization.
(i)~\emph{Resolution.} With binary rewards and $G=8$, a non-degenerate group has $\hat p\in\{1/8,\dots,7/8\}$, so the estimated factor $1/\sqrt{\hat p(1-\hat p)}$ lies in $[2.0,3.0]$: per-prompt normalization can reweight two prompts of a batch by at most $1.5\times$, whereas the factor $1/\sqrt{\bV}$ of the analysis is unbounded as $\bV\to0$.
(ii)~\emph{Optimizer.} AdamW is largely invariant to rescaling the gradient by a constant, so the factor shared by per-prompt and global std is absorbed and only the relative reweighting of (i) remains; a factor that varies across iterations, such as global std's, is not fully absorbed.
(iii)~\emph{Base accuracy.} With 68\% of prompts solved at iteration~0, the prompts that receive a gradient are a minority with similar success probabilities. The held-out run starts from a harder training set, where active prompts are more numerous and more heterogeneous.

\paragraph{Early lead of IS-GRPO.}
At iteration 50 IS-GRPO leads REINFORCE by 3.5\,pp, against an across-seed standard deviation of 2.3\,pp. This coincides with that every training prompt has been rolled out at least once: once solved prompts are identified they are no longer drawn, so each batch consists of prompts that carry a gradient.

%% ---------------------------------------------------------------------
\subsection{Trajectory Coefficients: Measurement Details}\label{app:coef_details}
 
\paragraph{Estimator.}
For the 7B IS-GRPO run (seed~1, \Cref{subsec:coefficients}) we take $s_{i,t}$ from the tracker of \Cref{app:setup_details}, i.e.\ $s_{i,t}=\sqrt{\hat p_{i}(1-\hat p_{i})}$ with population variance, smoothed by the EMA; it is bounded by $1/2$ as \Cref{ass:ass_3} requires. Prompts with $s_{i,t}=0$ are kept in the averages $\mu_{k,t}=\frac1n\sum_i s_{i,t}^k$. From these we compute $c_t=\frac52\mu_{1,t}$, $C(T)$, $D(T)=\frac2T\sum_{t<T}\mu_{2,t}\mu_{4,t}$ and $D_{\mathrm{IS}}(T)=\sum_{t<T}\mu_{4,t}\big/\sum_{t<T}\mu_{1,t}^{-1}$ exactly as defined in \Cref{subsec:convergence,subsec:isgrpo}. All coefficients are measured along the IS-GRPO trajectory, whereas \Cref{thm:reinforce_is} refers to $D(T)$ along REINFORCE's own trajectory; the comparison is therefore a same-trajectory approximation.

\paragraph{Biases.}
Three biases affect $s_{i,t}$ and do not have a common direction: sampling noise of $\hat p_i$ from eight rollouts spreads prompts with equal true $p_i$ apart and inflates the measured heterogeneity; $\sqrt{\hat p(1-\hat p)}$ from eight samples is biased low; and the tracker of a prompt lags the current policy until that prompt is drawn again. Recomputing $\mu_{k,t}$ from many rollouts per prompt at saved checkpoints would remove all three.
 
\paragraph{Two kinds of heterogeneity.}
The heterogeneity that develops along the run is dominated by the contrast between prompts whose last group was all-correct or all-incorrect ($s_{i,t}=0$) and active prompts. Per-prompt normalization cannot exploit this contrast, since a zero-variance group receives zero advantage under every scheme and normalization only reweights active prompts, whose estimated factors differ by at most $1.5\times$ (\Cref{app:7b_details}). The selection rule of IS-GRPO acts on exactly this contrast by not drawing prompts with $s_{i,t}=0$. This accounts for the pattern of \Cref{fig:acc_7b}b: the normalization schemes coincide while IS-GRPO separates.
 
\paragraph{Why heterogeneity lowers $C(T)$.}
By concavity of $p\mapsto\sqrt{p(1-p)}$ and Jensen's inequality, the average of $\sqrt{\pi_i(1-\pi_i)}$ over prompts is smaller than $\sqrt{\bar\pi(1-\bar\pi)}$ at the mean accuracy whenever per-prompt accuracies are heterogeneous. At any stage of training a substantial fraction of prompts sit near $\pi_i\approx0$ or $\pi_i\approx1$ and contribute little to the sum, which is why $C(T)$ falls well below the value the mean accuracy alone would suggest.

%% =====================================================================
\section{Details on Datasets and Training}\label{app:experimental_details}
 
\subsection{Datasets}
\textbf{GSM8K}~\citep{cobbe2021gsm8k}: grade-school math, partitioned into Easy ($N\!=\!4{,}695$) and Hard ($N\!=\!1{,}909$) using \texttt{Qwen2-7B-Instruct} as evaluator; used for 1.5B training, with GSM8K validation and MATH Level~1--2 for evaluation.
\textbf{MATH}~\citep{hendrycks2021measuring}: Level~3--5 and Level~1--5 (256 problems) for the two 7B runs; Level~1--2 for 1.5B evaluation.
\textbf{AMC, AIME}: held-out evaluation of the 7B Level~3--5 run.
 
\subsection{Evaluation}
1.5B: greedy pass@1 via regex extraction of the final answer. 7B: greedy decoding ($n\!=\!1$, temperature\,$=0$), on AMC for the Level~3--5 run and on the training prompts for the Level~1--5 run (\Cref{app:7b_details}); \Cref{tab:acc_7b} reports the results.
 
\subsection{Training Configuration}
 
\begin{table}[h]
    \centering\small
    \caption{Hyperparameters. For 7B, entries with ``/'' give the Level~1--5 run / the Level~3--5 run.}
    \label{tab:hyperparams}
    \setlength{\tabcolsep}{4pt}
    \begin{tabular}{lcc}
        \toprule
        & \textbf{1.5B} & \textbf{7B} \\
        \midrule
        Base model          & Qwen2.5-Math-1.5B & Qwen2.5-Math-7B \\
        Training data       & GSM8K Easy/Hard & MATH Lv.\,1--5 / Lv.\,3--5 \\
        Evaluation          & GSM8K val, MATH Lv.\,1--2 & training prompts / AMC \\
        LoRA $r$ / $\alpha$ & 64 / 128 & 32 / 32 \\
        LoRA target         & all linear & all linear \\
        Optimizer           & \multicolumn{2}{c}{AdamW ($\beta_1\!=\!0.9,\,\beta_2\!=\!0.999$)} \\
        Learning rate       & $1\!\times\!10^{-5}$ & $5\!\times\!10^{-6}$ \\
        LR sweep (1.5B)     & \multicolumn{2}{c}{$\{5\!\times\!10^{-7},\,1\!\times\!10^{-6},\,5\!\times\!10^{-6},\,1\!\times\!10^{-5},\,5\!\times\!10^{-5}\}$} \\
        Batch (prompts)     & 64 & 16 \\
        PPO mini-batch      & 32 & 16 \\
        Rollouts $G$        & 8 & 8 \\
        Iterations          & 500 & 500 / 400 \\
        Seeds               & 3 & 3 / 1 \\
        Hardware            & \multicolumn{2}{c}{H200 140GB} \\
        Framework           & \multicolumn{2}{c}{verl + vLLM} \\
        \bottomrule
    \end{tabular}
\end{table}
 
\textbf{LR sweep.}
For 1.5B we ran 5 learning rates $\times$ 3 strategies $\times$ 3 seeds $=45$ runs on GSM8K. Main-text figures use REINFORCE's optimal learning rate for all methods; per-method optimal results are qualitatively similar.

\end{document}